\documentclass{article}

\usepackage[preprint]{neurips_2026}

\usepackage{microtype}
\usepackage{graphicx}
\usepackage{subcaption}
\usepackage{booktabs} %

\usepackage{times}
\usepackage[utf8]{inputenc} %
\usepackage[T1]{fontenc}    %
\usepackage[unicode]{hyperref}       %
\usepackage{calc}
\usepackage{multirow}
\usepackage{amsfonts}       %
\usepackage{nicefrac}       %
\usepackage{xcolor}         %
\usepackage{natbib}
\usepackage{amsmath, amssymb, amsthm}
\usepackage{mathtools}
\usepackage{keytheorems} %

\usepackage{algcompatible}
\usepackage{algorithm}

\usepackage[capitalise]{cleveref}
\usepackage{adjustbox}
\usepackage{longtable}
\usepackage{titlesec}
\usepackage{pdflscape}
\usepackage{siunitx}
\usepackage{wrapfig} %
\usepackage{enumitem}

\newcommand{\myparagraph}[1]{\noindent\textbf{#1}\xspace}

\newkeytheorem{theorem,proposition,lemma,corollary}[style=plain]
\newkeytheorem{definition,assumption}[style=definition]
\newkeytheorem{remark}[style=remark]

\def\epsilon{\varepsilon}

\newcommand{\mech}{\mathcal{M}}

\newcommand{\accuracy}{\ensuremath{\operatorname{acc}}}
\newcommand{\wga}{\ensuremath{\operatorname{WGA}}}
\newcommand{\round}{\operatorname{round}}

\DeclareMathOperator{\sample}{\text{sample}}

\def\x{\mathcal{X}}%
\def\y{\mathcal{Y}}%
\def\z{\mathcal{Z}}%

\newcommand{\categorical}{\operatorname{Categorical}}

\newcommand{\N}{\mathcal{N}}%

\newcommand{\ex}{\mathbb{E}}
\newcommand{\cov}{\operatorname{Cov}}
\newcommand{\prob}{\mathbb{P}}
\newcommand{\var}{\operatorname{Var}}

\newcommand{\clip}{\operatorname{clip}}

\newcommand{\n}{\|}

\usepackage{xspace}
\makeatletter
\DeclareRobustCommand\onedot{\futurelet\@let@token\@onedot}
\def\@onedot{\ifx\@let@token.\else.\null\fi\xspace}

\newcommand{\eg}{{e.g}\onedot} 
\newcommand{\ie}{{i.e}\onedot}

\makeatother

\newcommand{\privHighlight}[1]{{\color{blue}{#1}}}

\newcommand{\privMathHighlight}[1]{{\mathcolor{blue}{#1}}}

\newcommand{\propHighlight}[1]{{\color{purple}{#1}}}
\newcommand{\propMathHighlight}[1]{{\mathcolor{purple}{#1}}}

\newcommand{\avglambda}{\bar{\lambda}}
\newcommand{\avgtheta}{\bar{\theta}}
\newcommand{\truegrad}{\nabla}
\newcommand{\stochgrad}{\hat{\nabla}}
\newcommand{\opttheta}{\theta^\star}

\newcommand{\DPSGD}{DP-SGD\xspace}
\newcommand{\method}{Adaptively Sampled and Clipped Worst-case Group Optimization\xspace}
\newcommand{\acronym}{ASC\xspace}
\newcommand{\zhou}{aZB\xspace}
\newcommand{\zhouprop}{\zhou-prop\xspace}
\newcommand{\lossreweighting}{DP Loss Reweighting\xspace}
\newcommand{\lossreweightingShort}{DP-LRW\xspace}

\newcommand{\WGA}{WGA\xspace}
\newcommand{\AVG}{AVG\xspace}

\newcommand{\nondpGroupReweight}{{\tt{Group-Reweight}}\xspace}
\newcommand{\dpGroupReweight}{{\privHighlight{\tt{DP-}}\tt{Group-Reweight}}\xspace}

\title{Adaptive Sampling and Clipping for \\ Private Worst-Case Group Optimization}

\author{%
  Max Cairney-Leeming\\
  Institute of Science and Technology Austria (ISTA)\\
  \texttt{max.cairney-leeming@ist.ac.at}
  \And
  Amartya Sanyal \\
  University of Copenhagen \\
  \texttt{amsa@di.ku.dk}
  \And
  Christoph H. Lampert \\
  Institute of Science and Technology Austria (ISTA) \\
  \texttt{chl@ist.ac.at}
}

\begin{document}

\maketitle

\begin{abstract}

A central requirement for the acceptance of machine learning methods for human-centric
tasks is that they should be \emph{fair}, in the sense that they should work comparably
well for individuals from different societal groups.
A second, equally important, requirement is that they should respect the \emph{privacy} of
user data.
While techniques exist to address each aspect in isolation, such as \emph{worst-case
group optimization} for the former and \emph{differentially private SGD} for the latter, these
are often at odds with with each other, and no practical method currently exists to enforce
both requirements simultaneously.
In this work, we overcome this problem and propose an algorithm for optimizing
the worst-case group accuracy in a differentially private way.
Our main contribution is \emph{\acronym (\method)}, which adaptively controls both
the sampling rate and the clipping threshold of each group's gradient contributions.
Thereby, it is able to reweight the training objective in favor of harder-to-learn
groups, while keeping the noise required to enforce privacy low enough to
preserve model utility.
Our experiments show that \acronym achieves substantially higher worst-case group
accuracy than prior work, without sacrificing overall average accuracy.

\end{abstract}

\section{Introduction}

Two key aspects of trustworthy machine learning are \emph{privacy} and \emph{fairness}, and it is natural to want a training process that achieves both at the same time. However, a number of papers have concluded that differentially private SGD (\DPSGD), the most common private learning algorithm, is less fair than standard SGD, which is of course not always fair itself. While these works have developed a variety of methods to privately achieve the same (un)fairness as SGD, they have not solved a more general question, which has been studied by just one prior work \citep{zhouDifferentiallyPrivateWorstgroup2024}: \emph{how do we learn a truly fair and private model?} In this work, we develop practical algorithms for this goal, and compare them both theoretically and empirically to prior work.

Many notions of \emph{algorithmic fairness} have been proposed~\citep{barocas-hardt-narayanan}. In this work, we adopt the notion of fairness as achieving \emph{highest possible worst-group accuracy}, also known as Rawls's max-min criterion \citep{Rawls1974Maximin}, which is a flexible measure, applicable to both binary and multi-class settings where the dataset is divided into known groups (\eg demographic categories). Optimising for this objective is a form of Distributionally Robust Optimization (DRO) \citep{levyLargeScaleMethodsDistributionally2020,sagawaDistributionallyRobustNeural2020}. %
Further, we take the widely used notion of \emph{differential privacy}~\citep{dworkDifferentialPrivacy2006} to guarantee that the privacy of user data is protected, and base our algorithms on the foundation of DPSGD \citep{abadi2016deep}, which uses gradient clipping and adds noise to updates to guarantee privacy.

Beginning with \cite{tranDifferentiallyPrivateEmpirical2021}, several works have identified mechanisms that could cause DPSGD's unfairness, as well as several different approaches to mitigate these effects, including extra regularization terms \cite{tranDifferentiallyPrivateEmpirical2021,tranDifferentiallyPrivateFair2021}, altered clipping schemes \cite{xuRemovingDisparateImpact2021}, and adaptive clipping \citep{esipovaDisparateImpactDifferential2022}. %
However, these approaches only seek to undo the unfairness caused by DPSGD's internal workings, and do not address problems where SGD does not give a fair outcome.
From a complementary direction, \cite{zhouDifferentiallyPrivateWorstgroup2024} introduce algorithms for SGD-based private worst-group optimization, and prove their convergence properties. However, they do not experimentally test their algorithms, and we have found the design to be ill-suited for practical use, as it does not easily support groups of different sizes, which is almost universal in real-world datasets. We also note a recent work on private DRO, \cite{xuDifferentiallyPrivateNonconvex2026}, which studies other variants of DRO where datasets do not have group structure.

We make the following contributions to this field:
\begin{enumerate}
  \item A new, easy-to-implement baseline based on loss reweighting.
  \item A practical algorithm for private worst-group optimization, based on adaptively setting the sampling and clipping parameters of \DPSGD.
  \item A study of convergence for stochastic, noisy worst-group optimization algorithms, and a comparison of sampling variance and convergence speed between different algorithms.
  \item The first experimental comparison of private worst-group optimization algorithms, showing that our algorithm achieves higher worst-group accuracy than baselines without sacrificing overall accuracy.
\end{enumerate}

\section{Worst-Group Optimization}

We study a \emph{supervised learning task with group information}, \ie
each data point is also annotated with a group label at training time.
Formally, for an input set $\x$, an output set $\y$, and a set of
groups, $[G]=\{1,\dots,G\}$, we assume a data distribution $P_g$
over $\z=\x\times\y$ for each group.
For training, a dataset $D_g\sim P_g$ of size $n_g$ is given
for each $g\in[G]$, which form the overall training set,
$D=\bigcup_{g=1}^G D_g$, of size $N=\sum_{g=1}^G n_g$.
At test time, group labels are not available to the model.

\myparagraph{Worst-case group optimization (WGO).}
For a class of models parametrized by $\theta \in \Theta$, let
$\accuracy(\theta, S)$ denote the accuracy of a model\footnote{For
conciseness, we will identify models and their parameters.} $\theta$ on a dataset $S$.
For any model $\theta$ we are interested not only in its quality on
the whole dataset, $\accuracy(\theta, D)$, but also in the lowest accuracy
it achieves on any group,
\begin{equation}
  \wga(\theta, D) := \min_{g\in[G]} \accuracy(\theta, D_g),
\end{equation}
which we call the \emph{worst-case group accuracy (WGA)}.

To train models for the highest $\wga$, we change the maximization of $\accuracy$, which is a discrete function, into a minimization over a continuous loss
function, $\ell(z, \theta)$, and we relax the minimization
over groups into a linear program over per-group weights. The final optimization problem is then given by:
\begin{equation}
  \min_\theta \max_{\lambda\in\Delta} \quad\sum_{g=1}^G \Big[ \frac{\lambda_g}{n_g}\!\sum_{z\in D_g}\ell(z,\theta)\Big],
  \label{eq:WGO}
\end{equation}
where $\Delta=\{\lambda\in[0,1]^G: \sum_g \lambda_g = 1\}$
is the probability simplex, and the solution to \cref{eq:WGO} will
identify the worst-case (highest loss) group because the objective function is a linear program in $\lambda$, so its optimum will
be attained at a corner of the simplex.

This approach is also known as \ group distributionally
robust optimization \citep{sagawaDistributionallyRobustNeural2020,zhangStochasticApproximationApproaches2023}, as it is equivalent to minimizing the
training loss not for a fixed data distribution, but across
all mixture distributions with components $(P_g)_{g\in[G]}$, parameterized by $\lambda$.
\begin{figure}[t]
  \begin{minipage}[t]{.49\linewidth}\centering
    \begin{algorithm}[H]
      \centering
      \caption{Non-private WGO }  \label{alg:non-dp-wgo}
      \begin{algorithmic}[1]
        \REQUIRE
        batch size $M$,
        learning rate $\eta$,

        \smallskip
        \STATE initialize model parameters $\theta_1\in \Theta$
        \STATE initialize per-group weights $\lambda_g=\frac{1}{G}$ for $g\in[G]$

        \FOR{$t = 1,\dots, T$}

        \STATE sample a batch $B$ of size $M$ from $D$ uniformly
        \STATE
        $\displaystyle \theta_{t+1} \leftarrow  \theta_{t} -
        \frac{\eta}{M}   \sum\limits_{z\in B}
        \lambda_{g(z)}\frac{N}{n_{g(z)}}\nabla\ell(z,\theta_t)    $
        \IF{$t \operatorname{mod} k = 0$}
        \STATE $\lambda \leftarrow \text{\nondpGroupReweight}(\lambda)$
        \ENDIF %
        \ENDFOR %
      \end{algorithmic}
    \end{algorithm}
  \end{minipage}
  \hfill
  \begin{minipage}[t]{.49\linewidth}
    \centering
    \begin{algorithm}[H]
      \caption{\dpGroupReweight($\lambda$)}  \label{alg:group-loss-reweighting}
      \begin{algorithmic}[1]
        \REQUIRE current weights $\lambda$;
        sampling rate $\gamma^\ell$,
        temperature $\eta^\ell$,
        \privHighlight{noise scale $\tau^2$,
        clipping threshold $\zeta$}
        \smallskip
        \FOR{$g=1,\dots, G$}
        \STATE sample a batch $B^\ell_g$ size $\lfloor\gamma^\ell \cdot n_g \rfloor$ from $D_g$ \label{line:zhou:loss-batches}

        \STATE $\displaystyle L_g\!\leftarrow\!\sum\limits_{z \in B^\ell_g}  \privMathHighlight{\clip\!\big(}\ell(z,\theta_{t})\privMathHighlight{, \zeta\big)}$
        
        \STATE $\displaystyle\tilde{\lambda}_g \leftarrow \lambda_g \exp \Bigg(\eta^\ell \frac{ L_g \privMathHighlight{+\N(0,\tau^2)} }{ |B^\ell_g|}\Bigg)$
        \ENDFOR
        \STATE \bf{return} $\displaystyle
        \tfrac{\tilde{\lambda}}{\|\tilde{\lambda}\|_1} \in \Delta_G$.

      \end{algorithmic}
    \end{algorithm}
  \end{minipage}
\end{figure}
For non-private settings, \Cref{alg:non-dp-wgo} and \cref{alg:group-loss-reweighting} provide a straightforward way to solve the worst-group optimization problem, based on the algorithms of \cite{zhangStochasticApproximationApproaches2023}. It consists of a parameter update from a minibatch, where each point's contribution is weighted by $\lambda$ and the group size, and a periodic update of the group weights using the losses of the model on each group.

\section{Differential privacy}
\label{sec:dp-def}
Following the seminal work of~\citet{dworkDifferentialPrivacy2006},
we consider an algorithm private, if one cannot determine with high probability
from its output if any data point was part of the input or not.
As our datasets are split by group, we give a slightly specialized formal definition of privacy, such that $\mech$, taking a dataset $D$ as input, and returning output   $o\in\mathcal{O}$ is differentially private if
$\prob(\mech(D)\in O)\leq e^\epsilon\prob(\mech(D')\in O) + \delta$ for all sets $O\subset\mathcal{O}$ and any two datasets $D$ and $D'$, which differ by replacing one element from any one group, \ie $D=\bigcup_{g=1}^G D_g$ and $D'=\bigcup_{g=1}^G D'_g$, where $D_g=D'_g$ for all $g\neq g^*$, and $D_{g^*}$ and $D'_{g^*}$ differ by one element. This definition was used previously by \cite{zhouDifferentiallyPrivateWorstgroup2024}, though we note that it does not protect the group membership of a datapoint, but its features and label.

In our analysis we also make use of \emph{Rényi Differential Privacy (RDP)}~\citep{mironovRenyiDifferentialPrivacy2017}, which substitutes
the above inequality by $D_{\alpha}(\mech(D)\|\mech(D'))\leq \epsilon$, where $D_{\alpha}(P,Q):=\frac{1}{\alpha-1}\log\mathbb{E}_{o\sim Q}\big(\frac{P(o)}{Q(o)}\big)^{\alpha}$ to define $\mech$ being $(\alpha,\epsilon)$-RDP, for $\alpha>1$.
Rényi DP offers several advantages over $(\epsilon,\delta)$-DP when analyzing SGD-like algorithms, and we detail its relevant properties in \cref{app:renyi}.

Private model training commonly relies on the \DPSGD algorithm~\citep{abadi2016deep}, given in
\Cref{alg:dp-sgd} in the appendix.
It follows the standard steps of mini-batch SGD, except for two modifications:
1) the function $\clip(x;\xi)=x\,\min(\frac{\xi}{\|x\|_2},1)$ is applied to
each computed gradient vector, ensuring that no data point contributes
more than a vector of norm $\xi$ to a model update, and
2) Gaussian noise of variance $\sigma^2$ is added to the accumulated
gradient vectors before using them for a model update.

We also note that privacy limits the contribution of any one group to a model update step. Using the limiting behavior shown by \cite{raisaSubsamplingNotMagic2024} to illustrate, let $\xi(q,T)$ be the best clipping threshold for a sampling rate of $q$ and $T$ steps and a fixed privacy budget. In the limit of $T \to \infty$, $\xi(q,T) = 1/q \cdot \xi(1,T)$, so the per-step contribution of group $g$ is at most $(q \cdot n_g) \cdot \xi(q,T) = n_g \xi(1,T)$, meaning that the contributions of a very small group to any \DPSGD-like algorithm are limited, and so it is possible, with extremely small groups that their updates are overwhelmed by other groups.

\section{Private WGO: Loss reweighting baseline}
\label{sec:loss-reweighting}
\begin{figure}[t]
  \begin{minipage}[t]{.49\linewidth}\centering

    \begin{algorithm}[H]
      \caption{\lossreweighting}  \label{alg:dp-loss-reweighting}
      \begin{algorithmic}[1]
        \REQUIRE %
        number of steps $T$,
        batch size $M$,
        learning rate $\eta$,
        gradient clipping threshold  $\xi$,
        noise scale $\sigma^2$,
        weight update frequency $k$
        \smallskip
        \STATE initialize model parameters $\theta_1\in \Theta$
        \STATE initialize group weights $\lambda_g=\frac{1}{G}$ for $g\in[G]$
        \FOR{$t = 1,\dots, T$}
        \STATE sample a batch $B$ of size $M$ from $D$ uniformly.
        \STATE // compute and clip the weighted gradients:
        \STATE $\displaystyle u \leftarrow \!\sum\limits_{z\in B}
        \clip\!\left( \lambda_{g(z)}N / n_{g(z)} \cdot \nabla\ell(z,\theta_t), \xi\right)$
        \STATE // update model parameters:
        \STATE    $\displaystyle \theta_{t+1} \leftarrow  \theta_{t} -
        \frac{\eta}{M} \Big( u + \N(0,  \sigma^2 I)    \Big)$ %
        \IF{$t \operatorname{mod} k = 0$}

        \STATE $\lambda \leftarrow \text{\dpGroupReweight}(\lambda)$
        \ENDIF %
        \ENDFOR %
      \end{algorithmic}
    \end{algorithm}

  \end{minipage}
  \hfill
  \begin{minipage}[t]{.49\linewidth}
    \centering
    \begin{algorithm}[H]
      \caption{\method (\acronym)}\label{alg:method}
      \begin{algorithmic}[1]
        \REQUIRE %
        number of steps $T$,
        batch size $M$, %
        learning rate $\eta$, %
        noise scale $\sigma^2$, %
        weight update freq. $k$,
        per-step privacy budget $(\alpha,\epsilon^s)$-RDP
        \smallskip
        \STATE initialize   $\theta_1\in \Theta$,
        $\lambda_g=\frac{1}{G}$ for $g\in[G]$
        \FOR{$t = 1,\dots, T$}
        \STATE set group batch sizes $m \leftarrow \round_M(\lambda M)$\label{line:meth:round-batch-sizes}
        \FOR{$g=1,\dots, G$} \label{line:meth:start-sampling}
        \STATE sample $B_g$ of size $m_g$ from $D_g$
        \label{line:meth:per-group-sampling}
        \STATE // compute and clip gradients:
        \\ \quad $\displaystyle u_g \leftarrow \!\sum_{z\in B_g}
        \clip\!\Big(\nabla\ell(z,\theta_t), \xi_{\sigma,\epsilon_s}\Big(\frac{m_g}{n_g}\Big)\Big)$ %
        \label{line:meth:per-group-clip-grads}
        \ENDFOR
        \STATE // update model parameters: 
        \STATE$\displaystyle \theta_{t+1} \leftarrow  \theta_t -
        \frac{\eta}{M} \Big( \sum_{g \in [G]} u_g + \N(0,  \sigma^2 I)    \Big)$
        \label{line:meth:model-update}
        \IF{$t \operatorname{mod} k = 0$}
        \STATE $\lambda \leftarrow \text{\dpGroupReweight}(\lambda)$
        \ENDIF %
        \ENDFOR %
      \end{algorithmic}
    \end{algorithm}
  \end{minipage}
\end{figure}
To design a private WGO algorithm, we can consider the parameter update steps (\Cref{alg:non-dp-wgo}) and the group reweighting part (\cref{alg:group-loss-reweighting}) individually. Indeed, \cref{alg:group-loss-reweighting} is easily privatized, shown by the clipping and noise addition steps given in blue, which implement a subsampled Gaussian mechanism. However, it is less obvious how we can privatize the main body of the algorithm, the weighted gradient descent, without the privacy changes affecting the weighting.

The key tenet of DPSGD is that it limits the contributions of each datapoint by clipping, to bound the sensitivity of the Gaussian mechanism that it implements. However, this is in direct contradiction to the aims of worst-group optimization, which wants to give bigger contributions to harder datapoints, as seen in the $\lambda_{g(z)}N/n_{g(z)}$ term, which rescales a point's contribution according to its group's relative size and importance. It is very likely that this term is above 1 for a minority group, but since clipping must be applied to the whole contribution from a datapoint, the clipping may undo this weighting. 

Despite this potential issue, \Cref{alg:dp-loss-reweighting} presents a good baseline for private worst-group optimization, as it preserves exactly the same privacy guarantees as \DPSGD (stated formally in \cref{lem:dp-loss-reweighting-privacy}), by keeping the same clipping and sampling behavior, while optimizing (with bias) for a non-uniform objective. Our later results will show that this algorithm performs better than \DPSGD (\cref{sec:results_accuracy}), as expected, but introduces significant variance in the updates (\cref{sec:variance-comparison}).

\section{Private WGO: Adaptive sampling}
\label{sec:asc-our-method}
However, loss reweighting is not the only approach to implement weighted gradient descent, and the fact that \cref{alg:dp-loss-reweighting} operates within the sampling structure of \DPSGD means that, among other things, the variance of the update from a minority group is significantly higher, as it is based on a smaller sample. This motivates a search for better approaches, and so we propose a deeper change to the \DPSGD structure: adaptively changing the sampling scheme to achieve the group weighting.

\textbf{Adaptive sampling}: The algorithm uses sampling to achieve the group weighting: out of the total batch size $M$, it samples $m_g := \lambda_g M$ points (rounded fairly, see \cref{app:our-method-details}) from group $g$, meaning that the expected update matches the objective's weighting (line \ref{line:meth:per-group-sampling}). These proportions will naturally change as $\lambda$ is updated.

However, without further changes to the privacy-related behaviour of the algorithm, this sampling scheme does not have a (reasonable) privacy guarantee, without assuming the worst-case that $\lambda_g=1$ for any groups could occur for any group at any time. To avoid this, we also adaptively set the clipping thresholds to guarantee a consistent, known and tight privacy guarantee, regardless of how the sampling changes:

\textbf{Balanced Clipping}: Each step of the parameter update is designed to achieve a $(\alpha,\epsilon^s)$-RDP privacy bound. As the sampling rate $m_g/n_g$ and the noise variance $\sigma^2$ are already chosen, we can set the clipping threshold, defined by the function $\xi_{\sigma,\epsilon_s}(m_g/n_g),$ to attain the desired privacy guarantee. This function is defined so that the following bound on the privacy of the step holds for all groups $g$:
\begin{equation}
  \epsilon_\alpha\left(\frac{m_g}{n_g}, \frac{\sigma}{\xi_{\sigma,\epsilon_s}(m_g/n_g)}\right)  \leq \epsilon^s,
\end{equation}
where $\epsilon_\alpha(\gamma, \beta)$ computes the Renyi DP guarantee of a  Gaussian mechanism with noise multiplier $\beta$ applied to a batch sampled at rate $\gamma$. Thus, clipping a point to norm $\xi_{\sigma,\epsilon_s}(m_g/n_g)$ (line \ref{line:meth:per-group-clip-grads}) ensures that the sensitivity $\frac{\sigma}{\xi_{\sigma,\epsilon_s}(m_g/n_g)}$  is not too high, and the privacy loss, including the effect of subsampling, is controlled.

In \cref{app:renyi} we formally define $\epsilon_\alpha$, which is a continuous, though relatively complex function. Thus, we define $\xi_{\sigma,\epsilon_s}(m_g/n_g)$ using numerical methods to invert $\epsilon_\alpha$. There is always at least one solution, given by the scenario without subsampling: $\epsilon_\alpha(1, \sigma/\xi) = \epsilon^s$, so $\xi = \sqrt{(\sigma^2 \epsilon^s)/ (2 \alpha)}$.
 
Note that the noise is added to the sum of the group-wise updates,
instead of adding noise to release each update separately, reducing the total amount of DP noise needed. This is possible because the algorithm is designed to modify the clipping thresholds, rather than setting fixed clipping thresholds and adaptively varying the noise variance.

We now establish the privacy guarantees of \Cref{alg:method}, where the full proof is given in \cref{app:our-method-details}.

\begin{theorem}\label{thm:method_dp}
  \Cref{alg:method} is $(\alpha,\epsilon)$-RDP where
  \begin{equation}
    \epsilon  = T \epsilon^s + \left\lfloor\frac{T}{k}\right\rfloor \epsilon_\alpha(\gamma^\ell, \tau/\zeta),
  \end{equation}
  for the predefined $(\alpha, \epsilon^s)$-RDP privacy budget for each step.
\end{theorem}
\begin{proof}[Proof sketch]
  The privacy guarantees for \cref{alg:method} arise from composing guarantees for each step. Firstly, each model update has been designed to give a guarantee of $(\alpha,\epsilon^s)$-RDP.

  If we take two neighbouring datasets $D = (D_g)_{g\in[G]}$ and $D' = (D'_g)_{g\in[G]}$, then they differ in exactly one group, say $\hat{g}$. For all $g \neq \hat{g}$, the update $u_g$ will be the same, whereas it will differ on $\hat{g}$, but the gradient from $D$ and the one from $D'$ will both be clipped to norm $\xi_{\sigma,\epsilon_s}(m_g/n_g)$. By the definition of the function $\xi$, the Rényi divergence between the outputs on $D$ and $D'$ is bounded by $(\alpha,\epsilon^s)$-RDP. This bound holds for all groups $\hat{g}$, and all neighboring datasets.

  Meanwhile, \dpGroupReweight satisfies a standard subsampled Gaussian privacy guarantee $\epsilon_\alpha(\gamma^\ell, \tau/\zeta)$-RDP, based on the sampling and clipping it performs.

  To combine these guarantees, we apply adaptive composition for $T$ steps. Although the algorithm adaptively chooses a mechanism on each step by selecting $m_g, \xi_g$ using $\lambda_g$, which is derived from data-dependent past losses, this composition bound is permissible as the losses are privately released.
\end{proof}

\newcommand{\zhoualg}{
  \begin{wrapfigure}{r}{0.5\textwidth}
  \vspace{-0.5cm}
  \begin{minipage}{0.5\textwidth}
    \begin{algorithm}[H]
  \caption{\zhou \propHighlight{-prop} ~\citep{zhouDifferentiallyPrivateWorstgroup2024}}  \label{alg:zhou}
  \begin{algorithmic}[1]
    \REQUIRE %
    number of iterations $T$,
    batch size $M$, %
    learning rate $\eta$, %
    gradient clipping threshold  $\xi$, %
    noise scale $\sigma^2$, %
    weight update frequency $k$
    \smallskip
    \STATE initialize model parameters $\theta_1\in \Theta$
    \STATE initialize per-group weights $\lambda_g=\frac{1}{G}$ for $g\in[G]$
    \FOR{$t = 1,\dots, T$}
    \STATE select a group $g \sim \categorical(\lambda)$.
    
    \STATE sample a batch $B$ of size $M \propMathHighlight{\cdot \frac{n_g}{N}}$  from $D_g$ uniformly
    \STATE // compute and clip gradients: \\
    $\displaystyle u \leftarrow \!\sum\limits_{z\in B}
  \clip\!\big(\nabla\ell(z,\theta_t), \xi)\big)$
  \STATE // update model parameters: \\
  $\displaystyle \theta_{t+1} \leftarrow  \theta_{t} -
  \frac{\eta}{|B|} \Big( u + \N(0,  \sigma^2 I)    \Big)$  
  
  \IF{$t \operatorname{mod} k = 0$}
  \STATE $\lambda \leftarrow \text{\dpGroupReweight}(\lambda)$
  \ENDIF %
  \ENDFOR %
  \smallskip\ENSURE trained model $\theta_{T+1}$
\end{algorithmic}
\end{algorithm}
\end{minipage}
\vspace{-1.3cm}
\end{wrapfigure} 
}

\section{Relation to Prior Work}

\label{sec:existing_methods}

Several works in the literature have studied the
problem of training models with max-min fairness
using distributionally robust optimization, \eg \citet{sagawaDistributionallyRobustNeural2020,duchiLearningModelsUniform2021}
introduce algorithms and \citet{barocas-hardt-narayanan}
discusses the broader relation to algorithmic fairness.

Learning with differential privacy has found even
more attention in the literature, both in classical machine learning ~\citep{ji2014differentialprivacymachinelearning}, and deep learning \citep{ponomareva2023,el2022differential,demelius2025}.
However,  research at their intersection is still at its infancy.

The prior work most related to ours is
\citet{zhouDifferentiallyPrivateWorstgroup2024},
where three possibilities to perform differentially
private WGO are discussed: two of these are not
applicable to the setting we study, because they
rely on techniques that are impractical for
general-purpose (deep) learning tasks: \emph{output
perturbation}~\citep{chaudhuri2011differentially}
and \emph{private online convex optimization}~\citep{jain2012}.
The third, \texttt{NoisySGD-MGR}, is based on \DPSGD and
therefore readily applicable to deep learning.
We compare it to \acronym conceptually  %
in this section, and experimentally in Section~\ref{sec:experiments}.

Like \acronym, \citet{zhouDifferentiallyPrivateWorstgroup2024}'s
algorithm rely on the general WGO framework seen in \cref{alg:non-dp-wgo}.
In contrast to \acronym, it does not vary the composition of
batches, but instead uses the group weighting $\lambda=(\lambda_g)_{g\in[G]}$ as probabilities in a categorical distribution, to select which group to sample a batch from. It then applies the same, non-adaptive, privatization steps as \DPSGD.

We present this algorithm in \cref{alg:zhou}, with two small changes to align it with standard practice in private deep learning: we introduce clipping on the gradients and loss values, as the original algorithm makes several restrictive assumptions to avoid implementing clipping. Firstly, it assumes that the loss function is both bounded and Lipschitz, and secondly that the algorithm runs projected descent on a bounded space, neither of which align with typical practice. We also add Gaussian noise in \cref{alg:group-loss-reweighting} instead of Laplacian noise, which reduces the tails of the noise distribution, and makes the composition in the privacy analysis simpler. 
\zhoualg
A more significant concern with this algorithm is that it assumes the groups are balanced, \ie have the same number of data points. This assumption is unrealistic in practice, and our privacy analysis of the algorithm (\cref{app:zhou-priv-proof}) shows that the guarantee depends on the size of the smallest group, which is often significantly smaller than the average group size or even the dataset size. To address this, we have also introduced a variant of the algorithm, \zhouprop, that uses proportional sampling to give equal privacy guarantees to all groups, at the cost of smaller batches and increased variance for smaller groups. We include both \zhou and \zhouprop in our experiments.

\section{Sampling Variance Analysis}
\label{sec:variance-comparison}

All of training algorithms we have discussed follow an
SGD-style\footnote{Our algorithm could readily be phrased for other optimizers, such as Adam, but this is beyond the
scope of our work.} scheme, with common elements such as noise added for privacy, and group reweighting. Thus, we can consider their convergence properties jointly by analyzing a generic algorithm (presented formally in \cref{alg:generic-noisy-wgo}) which performs WGO with noise added to both the model updates and the group weight updates. This model does not capture all aspects of these algorithms because clipping is not included. Clipping is a significant added complication, as the group weightings $\lambda$ also affect the relationship between gradient norms and clipping for \lossreweightingShort and \acronym. %
Our theorem, stated informally here, and in full detail in \cref{thm:genericConvergenceThm}, is as follows:

\begin{theorem}[name=Convergence of noisy WGO (informal)]
  Consider the application of a noisy WGO algorithm (\cref{alg:generic-noisy-wgo}) to a convex, bounded $L$-smooth loss function,  and a bounded weight space. Given that at time $t$, with model parameters $\theta_t$ and group weights $\lambda_t$, $U_t^\text{WGO}$ is the non-stochastic update, and the algorithm's unbiased stochastic update is $U_t^{\text{noisy}}$, we assume that the variance $\ex \n U_t^{\text{noisy}} - U_t^\text{WGO} \n_2^2$ is bounded by $v_\mathrm{sample}^2$ uniformly in $t, \theta_t,$ and $\lambda_t$. If the algorithm uses DP noise with variance $\sigma_\text{DP}^2$, it enjoys the following convergence rate, eliding the problem-specific factors:

  \begin{equation}
    \begin{aligned}
      \ex[\max_g \mathcal{L}(D_g, \bar{\theta})] - \min_\theta \max_g \mathcal{L}(D_g, \theta) \leq  &O\left(\frac{v_\mathrm{sample} + \sigma_\text{DP} + \tau \sqrt{ \log (N^2 T)}}{\sqrt{T}}\right)   %
    \end{aligned}
  \end{equation}

\end{theorem}

As a result, the clearest differentiation between algorithms on the same problem is their sampling variance $v_\mathrm{sample}^2$, which aligns with prior work on both traditional \citep{ghadimi2013stochastic} and information-theoretic \citep{NEURIPS2021_cb77649f} convergence guarantees for SGD-like algorithms.
In \cref{thm:sampling-variances}, we present a precise characterisation of the sampling variances for each algorithm, at an arbitrary timestep $t$, where the model has parameters $\theta_t$ and group weights $\lambda_t$. We decompose these terms, using the current per-group sampling variance as $\var_{g} := \ex_{z\sim D_g}\big\| \nabla_\theta\ell(z, \theta_t) - U(D_g) \big\|_2^2$, and the inter-group discrepancy $\|U(D_g) - U_t^\text{WGO}\|_2^2$. As is common in convergence analsyis \citep{koloskovaRevisitingGradientClipping2023,9683272}, we assume there are bounds on these quantities,  $v^2$ and $\Delta$, which are independent of time, parameters, groups, and weightings, which allow us to  give simpler bounds for easier comparison:

\begin{theorem}[name=Comparison of sampling variances,store=samplingVarianceThm]
  \label{thm:sampling-variances}
  Given these definitions, we can express the sampling variance of the algorithms as follows:
  \begin{align}
    \label{eq:our-sampling-var-in-theorem}
    \var(U^{\text{\acronym}}) &= \frac1M \sum_{g \in [G]}\lambda_g \frac{n_g - m_g}{n_g - 1} \var_g
    &&\leq \frac1M v^2.
    \intertext{In contrast, \zhou has sampling variance:}
    \label{eq:zhou-bass-sampling-var-in-theorem}
    \var(U^{\text{\zhou}})
    &=  \sum_{g \in [G]}\lambda_g \left(\frac1M\frac{n_g - M}{n_g - 1} \var_g +  \|U(D_g) - U_t^\text{WGO}\|_2^2  \right)
    &&\leq \frac{1}{M}  v^2 + \Delta,
    \intertext{\zhouprop, which samples a batch of size $M_g = M n_g /N$ for group $g$, has sampling variance:}
    \label{eq:zhou-prop-sampling-var-in-theorem}
    \var(U^{\text{\zhouprop}})
    &= \sum_{g \in [G]} \lambda_g \left(\frac{1}{M_g}\frac{n_g - M_g}{n_g - 1} \var_g +   \|U(D_g) - U_t^\text{WGO}\|_2^2  \right)
    &&\leq \frac{1}{\min_g M_g}  v^2 + \Delta,
  \end{align}
  Finally, the variance of \lossreweighting's sampling scheme defies easy characterization using $\Delta$, as its updates are distorted by the rescaling applied to the losses:
  \begin{align}
    \label{eq:reweight-sampling-var-in-theorem}
    \var(U^{\text{\lossreweightingShort}}) &= \frac{1}{M}\frac{N - M}{N - 1} \left( \sum_{g \in [G]} N \frac{\lambda_g^2}{n_g} \var_g + \sum_{g \in [G]} \frac{n_g}{N} \left\| \frac{N \lambda_g}{n_g} U(D_g) - U^{\text{WGO}} \right\|_2^2 \right) \\
    &\leq \frac{1}{M} \frac{N}{\min_g n_g} v_\text{group}^2 + \frac{1}{M} \sum_{g \in [G]} \frac{n_g}{N} \left\| \frac{N \lambda_g}{n_g} U(D_g) - U^{\text{WGO}} \right\|_2^2.
  \end{align}
\end{theorem}

Comparing the results in \Cref{thm:sampling-variances} (see \cref{app:sampling_variance_proof} for the proof),
one sees that \acronym has the simplest expression:
its sampling variance is simply an importance-weighted
combination of the individual within-group variances. An analog term appears for both \zhou and \zhouprop (with slight scaling differences as all methods use without-replacement sampling), but they also contain an additional term in $\Delta$ for the disparity between the group-specific update and the overall WGO update, which does not shrink as the batch size increases. This term is not present in \acronym because it samples from all groups in each step, and thus the sampling variance is only due to the within-group variance, and not the between-group variance.
This suggests that \acronym exhibits more
stable convergence behavior than \zhou and \zhouprop. \lossreweightingShort also has a more complex variance expression, due to its use of rescaling, though it does decrease as the batch size increases. We show experimental evidence of the differences in sampling variance in \Cref{sec:results_variance}.

\section{Experimental Evaluation}\label{sec:experiments}
\begin{table*}[t]
  \caption{Worst-case group accuracy (\WGA) and average group accuracy (\AVG) on the Unbalanced MNIST, CelebA and Bank Fraud (base variant) datasets. Bold text marks the best results in each column as well as any other methods that fall within one standard deviation of the best method.
  }
  \label{tab:headline_combined}
  \begin{adjustbox}{center}
    \begin{tabular}{lllllll}
\toprule
 & \multicolumn{2}{c}{Unbalanced MNIST} & \multicolumn{2}{c}{CelebA} & \multicolumn{2}{c}{Bank Fraud} \\
 & \WGA & \AVG & \WGA & \AVG & \WGA & \AVG \\
\midrule
\acronym (ours) & $\mathbf{75.2 \pm {\scriptstyle 1.2}}$ & $\mathbf{86.3 \pm {\scriptstyle 0.2}}$ & $\mathbf{64.4 \pm {\scriptstyle 2.9}}$ & $\mathbf{75.9 \pm {\scriptstyle 0.9}}$ & $\mathbf{46.3 \pm {\scriptstyle 2.0}}$ & $\mathbf{57.8 \pm {\scriptstyle 0.5}}$ \\
\lossreweightingShort & $62.8 \pm {\scriptstyle 2.2}$ & $74.7 \pm {\scriptstyle 0.5}$ & $28.3 \pm {\scriptstyle 0.6}$ & $72.5 \pm {\scriptstyle 0.3}$ & $\mathbf{44.9 \pm {\scriptstyle 2.4}}$ & $56.2 \pm {\scriptstyle 1.5}$ \\
DPSGD & $20.4 \pm {\scriptstyle 3.3}$ & $84.8 \pm {\scriptstyle 0.3}$ & $16.7 \pm {\scriptstyle 0.8}$ & $67.8 \pm {\scriptstyle 0.3}$ & $0.9 \pm {\scriptstyle 0.5}$ & $50.0 \pm {\scriptstyle 0.1}$ \\
DPSGD-F & $5.5 \pm {\scriptstyle 1.6}$ & $81.6 \pm {\scriptstyle 0.9}$ & $8.8 \pm {\scriptstyle 3.5}$ & $50.9 \pm {\scriptstyle 1.3}$ & $3.8 \pm {\scriptstyle 1.0}$ & $49.5 \pm {\scriptstyle 0.3}$ \\
\zhou & $1.3 \pm {\scriptstyle 0.2}$ & $13.1 \pm {\scriptstyle 0.8}$ & $4.6 \pm {\scriptstyle 4.1}$ & $51.2 \pm {\scriptstyle 1.1}$ & $39.9 \pm {\scriptstyle 1.9}$ & $53.7 \pm {\scriptstyle 1.4}$ \\
\zhouprop & $7.5 \pm {\scriptstyle 2.4}$ & $22.0 \pm {\scriptstyle 1.7}$ & $34.0 \pm {\scriptstyle 7.8}$ & $64.2 \pm {\scriptstyle 2.7}$ & $35.7 \pm {\scriptstyle 4.3}$ & $54.3 \pm {\scriptstyle 1.3}$ \\
\bottomrule
\end{tabular}

  \end{adjustbox}
\end{table*}

In this section, we report on experimental results of our proposed \acronym method, the naive approach, \lossreweighting, \DPSGD, \DPSGD-F, and two variants of \zhou, as discussed in \cref{sec:existing_methods}. \DPSGD-F \citep{xuRemovingDisparateImpact2021} adaptively adjusts the clipping threshold of each group to reduce the disparate impact of \DPSGD on per-group accuracy.
We use three datasets for evaluation: Unbalanced MNIST, CelebA, which are both very commonly used in the literature on private and/or fair learning, as well as the Bank Account Fraud dataset. Unbalanced MNIST has had class 8 downsampled, to simulate under-represented classes, whereas the CelebA dataset (where we predict \emph{male/female}) has a significant imbalance in photos of blond and non-blond male and female subjects, so we take these as the groups. Finally, the Bank Account Fraud dataset is a large-scale tabular dataset for binary classification of \emph{fraud/not fraud}. The groups are defined as the pairwise combinations of \emph{fraudulent}/\emph{not fraudulent} with \emph{age under/over 50}, as it is known that this demographic characteristic affects the likelihood of fraudulent behavior. We give details on these datasets and our experimental setup in \cref{app:experiments}, including the hyperparameters used for each method and dataset.

In all cases, we aim for models that are $(\epsilon=1,\delta=\frac{1}{2n})$-DP,
for which the privacy accounting library internally constructs a suitable RDP level
and determines the noise strengths for all methods. Note that for \zhou, setting a batch size of $512$ on Unbalanced MNIST must be reduced to $474$ as the smallest group (class 8) only has $474$ examples, and we cannot oversample.
\begin{figure*}[t]\addtolength{\abovecaptionskip}{-2pt}
  \begin{center}
    \includegraphics{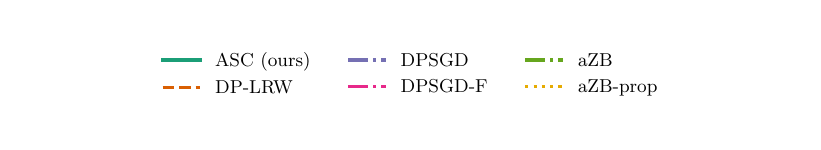}
    \vspace{-2.75em}
  \end{center}
  \begin{minipage}[t]{.653\linewidth}\centering
    \includegraphics{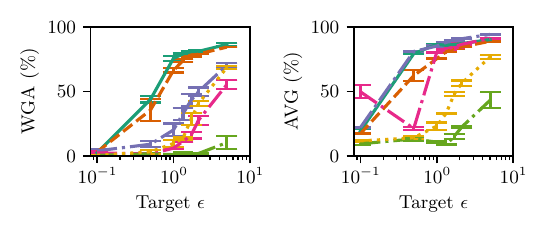}
    \vspace{-1.75em}

    \caption{Results on Unbalanced MNIST dataset for varying levels of privacy (x-axis).
      Left: worst-case group accuracy (\WGA). Right: average group accuracy (\AVG).
    }
    \label{fig:wga_eps}
    \vfill
  \end{minipage}
  \hfill
  \begin{minipage}[t]{.323\linewidth}\addtolength{\abovecaptionskip}{-8pt}
    \centering
    \includegraphics{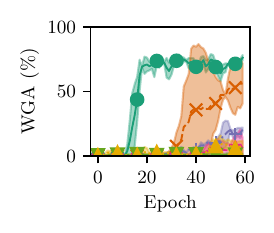}
    \caption{Validation set WGA during training on Unbalanced MNIST, $(\epsilon=1,\delta=\frac{1}{2n})$. }%
    \label{fig:wga_variance_mnist}
  \end{minipage}
  \vspace{-1em}
\end{figure*}

\myparagraph{Average and Worst-case Group Accuracy:} \label{sec:results_accuracy}
To evaluate the different methods, we measure the quality of the resulting model in terms of their worst-case group accuracy (WGA), as well as their average-across-groups accuracy (AVG) on held-out test data, given in \Cref{tab:headline_combined}.

Firstly, we have similar results for \DPSGD as in prior work~\citep{bagdasaryan2019differential}, showing that it achieves reasonable average accuracy ($84.8\%$ on Unbalanced MNIST, $67.8\%$ on CelebA), but very poor worst-case group accuracy ($20.4\%$ on Unbalanced MNIST, $16.7\%$ on CelebA). However, on the Bank Fraud dataset, \DPSGD performs poorly, with an average accuracy of $50\%$, equal to chance, and a worst-case group accuracy of only $0.9\%$, as it predicts almost all inputs as not fraudulent.

We also report the first experimental results on the algorithm of \citet{zhouDifferentiallyPrivateWorstgroup2024}, with mixed results. On Unbalanced MNIST, both \zhou and \zhouprop perform worse than DPSGD, though on CelebA \zhouprop does improve over DPSGD, and on Bank Fraud both improve significantly ($+35\%$) over DPSGD. These issues likely arise from the sampling strategy of \zhou and \zhouprop, which allow for convergence proofs, but are not so well suited to practical use. \zhou samples the same number of points in total as DPSGD, but in such a way that the amplification by subsampling is $N/\min_g n_g$ times weaker, requiring a much higher noise multiplier. \zhouprop avoids this issue by sampling smaller groups, and so learns from many fewer datapoints than the other private methods.

However, the results of \lossreweighting already show a much more positive situation, achieving \WGA improvements over DPSGD for all datasets, with positive effect on the average accuracy (except on Unbalanced MNIST, where it loses $\sim 10\%$).

Finally, we can see that \acronym improves over all other algorithms on all three datasets, achieving the best \WGA and \AVG in all cases. The improvements are particularly strong on the CelebA dataset, where \acronym achieves a \WGA $30\%$ higher than the next best method.

\myparagraph{Varying privacy levels:} \label{sec:results_priv_level}
In \cref{fig:wga_eps}, we study the effect of changing the privacy level, on a range from $\epsilon=0.1$ to $\epsilon=5$. All models perform poorly at the stronger epsilon, but  we can see improved performance as $\epsilon$ increases for all methods. \acronym and \lossreweighting maintain strongest, and second strongest, respectively \wga performance across all $\epsilon$ values.
Notably, we can also see much stronger performance from DPSGD-F as $\epsilon$ increases above 1, which is more in line with the results of the original paper \citep{xuRemovingDisparateImpact2021}, which trains at $\epsilon=6.55$.
This provides further evidence that \acronym is the most practical algorithm for training max-min fair classifiers with differential privacy across a range of privacy budgets.

\begin{figure*}
  \centering
  \includegraphics{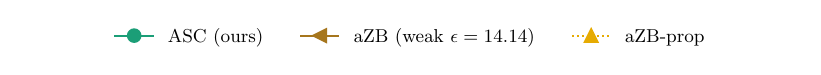}
  \vspace{-1.5em}

  \includegraphics{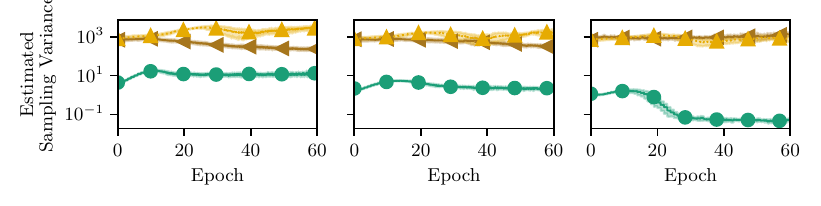}
  \vspace{-2em}

  \begin{subcaptiongroup}
    \centering
    \parbox[t]{.1\textwidth}{%
      \centering  \phantom{spacer}%
    }%
    \parbox[t]{.3\textwidth}{%
    \centering \caption{Batch size $M=128$}}%
    \parbox[t]{.275\textwidth}{%
    \centering \caption{Batch size $M=256$}}%
    \parbox[t]{.3\textwidth}{%
    \centering \caption{Batch size $M=474$}}%
  \end{subcaptiongroup}
  \caption{Estimated sampling variance of each method based on \cref{eq:our-sampling-var-in-theorem,eq:zhou-bass-sampling-var-in-theorem,eq:zhou-prop-sampling-var-in-theorem}
    Shaded regions indicate the standard deviation over 10 runs. Each run was repeated at 3 different batch sizes, 128, 256, and 474 (the size of the smallest group). $\epsilon=2.12$ for \acronym and \zhouprop, $\epsilon=14.14$ for \zhou.
  }\label{fig:sampling_variance_for_algs}
\end{figure*}

\myparagraph{Results -- Variance.}\label{sec:results_variance}
Empirical evaluations confirm our theoretical analysis (\cref{sec:variance-comparison}) that \acronym achieves significantly lower sampling variance than \zhou and \zhouprop. On the Unbalanced MNIST dataset (\Cref{fig:sampling_variance_for_algs}), \acronym exhibits an approximate 100-fold reduction in sampling variance compared to the baselines, which also decreases as the batch size increases. In contrast, the estimated variance for the \zhou variants fails to decrease with larger batches because their inter-group term (\cref{eq:zhou-bass-sampling-var-in-theorem,eq:zhou-prop-sampling-var-in-theorem}) is independent of batch size. Consequently, \acronym demonstrates more stable training dynamics with tighter uncertainty regions, as seen in \cref{fig:wga_variance_for_sampling_variance_runs} for the same three methods, and more generally in \cref{fig:wga_variance_mnist}.

\section{Conclusion and future directions}

In this work, we studied private worst-group optimization, in order to train models that are private and have high accuracy for all demographic groups within a dataset. Our new baseline, \lossreweightingShort, showed itself to be a simple and effective improvement over DPSGD, with the drawback that its sampling scheme has high variance introduced by the rescaling of the gradients.

However, our newly proposed algorithm, \acronym, is able to achieve stronger performance in our experiments by using a sampling scheme that mixes groups together in each batch using adaptive stratified sampling, while still ensuring that the privacy guarantees are consistent across all groups by adapting the clipping thresholds.

Our convergence analysis identified the variance of the update steps as a clear differentiator between algorithms, and we also showed that \acronym's sampling scheme has lower variance than those of existing methods. In the future, we plan to extend this analysis to include the effect of clipping.

Of course, there are some fundamental limitations caused by privacy that limit the contribution of very small groups to the model updates, as we discussed in \cref{sec:dp-def}, though we did not encounter this in practice, as our algorithm and also other private worst-group optimization algorithms were still able to significantly improve over DPSGD.

Overall, we believe that the societal impacts of this work are positive, as it provides a practical method for training fair and private machine learning models, and we do not foresee any negative consequences arising specifically from this work.
In future work, we plan to generalize the mechanism to other forms of distributionally robust optimization, \eg in the context of multi-task or meta-learning, where multiple models are learned,
and tasks can vary greatly in terms of their dataset sizes and desired privacy guarantees.

\section*{Acknowledgements}
This research was funded in whole or in part by the Austrian Science Fund (FWF) 10.55776/COE12.

This research was supported by the Scientific Service Units (SSU) of the Institute of Science and Technology Austria through resources provided by Scientific Computing (SciComp).

We acknowledge EuroHPC Joint Undertaking for awarding us access to MareNostrum5 at BSC, Spain.

AS acknowledges the Novo Nordisk Foundation for support via the Startup grant (NNF24OC0087820) and VILLUM
FONDEN via the Young Investigator program (72069).

\bibliography{strings,bib_automatic_from_zotero,bib_manual_entries}

\newpage
\appendix

\crefalias{section}{appendix}

\section{Renyi DP}
\label{app:renyi}
 
\begin{algorithm}[t]
  \caption{DPSGD~\citep{abadi2016deep}}\label{alg:dp-sgd}
  \begin{algorithmic}[1]
    \REQUIRE dataset $D$, number of iterations $T$, batch size $M$,
    learning rate $\eta$, clipping threshold  $\xi$, noise scale $\sigma^2$
    \smallskip\STATE initialize model parameters $\theta_1\in \Theta$
    \FOR{$t = 1,\dots, T$}
    \STATE sample a batch $B$ of size $M$ from $D$ uniformly %
    \STATE compute gradients and clip their length: \\
    $$u_{t} \leftarrow \sum_{z\in B}
    \operatorname{clip}\big(\nabla\ell(z,\theta_t), \xi\big)$$

    \STATE update model parameters and add noise:\\
    $$\theta_{t+1} \leftarrow \theta_{t} -
    \frac{\eta}{M} \Big( u_t + \mathcal{N}(0,\sigma^2 \text{I}) \Big)$$
    \ENDFOR
    \smallskip\ENSURE trained model $\theta_{T+1}$ %
  \end{algorithmic}
\end{algorithm}
\begin{definition}[Rényi Divergence, \citep{renyiMeasuresEntropyInformation1961}]
  The Rényi divergence of order $\alpha > 1$ between two probability distributions $P$ and $Q$ is
  \begin{equation}
    D_\alpha (P \| Q) = \frac{1}{\alpha-1} \log \ex_{x \sim Q} \left(
      \frac{P(x)}{Q(x)}
    \right)^{\alpha}
  \end{equation}

\end{definition}

Using this divergence, \citet{mironovRenyiDifferentialPrivacy2017} defines a new form of privacy as follows:
\begin{definition}
  [$(\alpha,\epsilon)$-RDP]
  A randomized mechanism $\mech$ that takes an input dataset $S$ and outputs values $o \in \mathcal{O}$ is said to be $\epsilon$-Rényi differentially private at order $\alpha$ if, for all neighbouring $D \sim D'$ it holds that
  \begin{equation}
    D_\alpha (\mech(D) \| \mech(D')) \leq \epsilon.
  \end{equation}

\end{definition}

We introduce the RDP guarantee for the Gaussian mechanism applied to a sum, under our neighbouring relation:

\begin{proposition}[RDP guarantee for the Gaussian mechanism (Proposition 7 \cite{mironovRenyiDifferentialPrivacy2017})]
\label{prop:gauss-rdp}
Define $f(D)=\sum_{z \in D} \clip(z, C)$, where $\clip$ bounds the $\|\cdot\|_2$ norm of $z$ to less than or equal to $C$. Under the 'replace one' neighboring definition, the sensitivity of this function is $2C$.
A Gaussian mechanism $\mech(D) = f(D) + \mathcal{N}(0, \sigma^2 I)$ is defined to have noise multiplier $\kappa:=\sigma/C$ (using the same notation as other papers) and a RDP privacy guarantee given by:

\begin{equation}
    \epsilon^\text{Gauss}_\alpha(\kappa) = \frac{\alpha}{2\cdot (\kappa/2)^2}
\end{equation}

\end{proposition}

And these guarantees can also be easily converted to $(\epsilon,\delta)$-DP. Note that we also follow the common practice of tracking the RDP guarantees of a mechanism at multiple orders ($\alpha$), and take the best resulting $(\epsilon,\delta)$-DP guarantee (measured by lowest $\epsilon$).

\begin{proposition} [Conversion of RDP to $(\epsilon,\delta)$-DP (Proposition 3 \cite{mironovRenyiDifferentialPrivacy2017}]
\label{prop:conv-rdp-eps}
  If $\mech$ is $(\alpha,\epsilon)$-RDP, then it satisfies
  $\left(\epsilon + \frac{\log(1/\delta)}{\alpha-1}, \delta \right)$ for any $\delta \in (0,1)$.
\end{proposition}

\citet{wangSubsampledRenyiDifferential2018} proves (Theorem 9) gives privacy guarantees for the application of a $(\alpha,\epsilon)$-RDP mechanism $\mech$ to a batch of size $m$ sampled without replacement from a full dataset of size $n$, \ie giving an effective sampling rate $m/n$:

\begin{theorem} [RDP for subsampled mechanisms, \citet{wangSubsampledRenyiDifferential2018}[Theorem 9]]
\label{thm:rdp-subs}
  Given a dataset of $n$ points drawn from a domain $\mathcal{X}$ and a mechanism $\mech: \mathcal{X}^m \to \mathcal{O}$ for a batch size  $m \leq n$, we define the subsampled mechanism $mech\circ\operatorname{subsample}$.
  First, this mechanism subsamples a batch of size $m$ without replacement and then applies $\mech$ to that batch.
  Let the sampling rate be  $\gamma = m/n$.

  For all integers $\alpha \geq 2$, if $\mech$ obeys $(\alpha,\epsilon_\alpha)$-RDP,  then this new randomized algorithm $\mech \circ \operatorname{subsample}$ obeys $(\alpha,\epsilon'(\alpha))$-RDP where,

  \begin{align*}
    \epsilon^\text{subs}_\alpha(\epsilon_{(\cdot)},\gamma)  = \frac{1}{\alpha-1}\log\bigg( 1
      +  \gamma^2{\alpha \choose 2} \min\Big\{ 4(e^{\epsilon_2}-1), e^{\epsilon_2} \min\{2, (e^{\epsilon_\infty}-1)^{2} \} \Big\}& \\
      +   \sum_{j=3}^{\alpha} \gamma^j {\alpha \choose j} e^
    {(j-1)\epsilon_j} \min\{2, (e^{\epsilon_\infty}-1)^{j}\}& \bigg).
  \end{align*}
\end{theorem}

\section{Further details for \lossreweightingShort \cref{alg:dp-loss-reweighting}}

\begin{lemma}[Privacy guarantees of \cref{alg:dp-loss-reweighting}]
  \label{lem:dp-loss-reweighting-privacy}
  \cref{alg:dp-loss-reweighting} has a privacy guarantee stated in Renyi DP given by:  $(\alpha,
  \epsilon_\text{\lossreweightingShort})$ given by
  \begin{equation}
    \epsilon_\text{\lossreweightingShort} =  
    T \cdot \epsilon_\alpha\left( \frac{M}{N}, \frac{\sigma}{\xi}\right)
    +
    \frac{T}{k} \cdot \epsilon_\alpha\left(\gamma^\ell, \frac{\tau}{\zeta}\right),
  \end{equation}
  which corresponds exactly to the privacy guarantee of $T$ steps of DPSGD with sampling rate $M/N$ and noise multiplier $\sigma/\xi$, and $T/k$ steps of the subsampled Gaussian mechanism with sampling rate $\gamma^\ell$ and noise multiplier $\tau/\zeta$.
\end{lemma}
\begin{proof}
  This guarantee follows from the privacy guarantees of the subsampled Gaussian mechanism, where we note that clipping is applied to the term including the reweighting $\lambda_g N/n_g$, so the noise multiplier is $\sigma/\xi$ regardless of the value of $\lambda_g$ or $n_g$. Thus, each model update step has a privacy guarantee of $\epsilon_\alpha(M/N, \sigma/\xi)$-RDP. Each update of $\lambda$ has a privacy guarantee of $\epsilon_\alpha(\gamma^\ell, \tau/\zeta)$-RDP. To combine these guarantees, we apply adaptive composition of these mechanisms. 
\end{proof}

\section{Further details for \acronym \cref{alg:method}}

\label{app:our-method-details}

\paragraph{Notation for RDP guarantees} In \cref{sec:asc-our-method}, we introduced the functions $\epsilon_{\alpha}(\gamma,\beta)$ and $\xi_{\sigma,\epsilon_s}(m_g/n_g)$, which we now define in more detail, using \cref{prop:gauss-rdp,thm:rdp-subs}:

\begin{align}
    \epsilon_\alpha(\gamma, \beta) \,=\; 
      &\epsilon^\text{subs}_\alpha(\epsilon^\text{Gauss}_{(\cdot)}(\beta),\gamma). \\
      \intertext{
      Expanding both definitions, and noting that the Gaussian mechanism is not private at order $\alpha=\infty$, this gives:
      } \\
    =\; &\frac{1}{\alpha-1}\log\bigg( 1
      +  \gamma^2{\alpha \choose 2} \min\Big\{ 4(e^{\frac{2 \cdot 2}{\beta^2}}-1), 2e^{\frac{2\cdot 2}{\beta^2}} \Big\} 
      +   \sum_{j=3}^{\alpha} \gamma^j {\alpha \choose j} 2 e^
    {(j-1)\frac{2j}{\beta^2}}   \bigg).
\end{align}
 
This function is then implemented by the privacy accountant library. The inverse function $\xi_{\sigma,\epsilon_s}(\gamma)$ is then defined by

\begin{equation}
    \xi_{\sigma,\epsilon_s}(\gamma) = \max\left\{ \kappa: \epsilon_\alpha(\gamma, \sigma/\kappa) \leq \epsilon_s \right\},
\end{equation}
and implemented using Brent's method. A suitable bracket $(\kappa_l, \kappa_h)$ is found by doubling/halving candidate values until they contain a root of the function $\kappa \mapsto \epsilon_\alpha(\gamma,\kappa') - \epsilon$. Even if this does not find a perfect minimizer, it will give a valid $\kappa$ value, in the sense of satisfying $(\alpha,\epsilon)$-RDP, as the lower half of the bracket is always returned. As noted in the main text, $\epsilon_\alpha$ is a continuous function, and an lower bracket can always be found, by solving for the clipping threshold in a scenario without subsampling.

\paragraph{Rounding}

In line \ref{line:meth:round-batch-sizes} of \cref{alg:method}, the per-group batch sizes are normalized and rounded, with the three aims that: the proportions stay the same, all the batch sizes are integer, and their sum is $M$. All three requirements cannot be jointly satisfied, so the exact proportionality must be relaxed. First, we rescale the batch sizes to sum up to $M$, and round to the nearest integer (using bankers' rounding). If the resulting sum is not $M$, we randomly distribute the difference uniformly among the elements in the array. Of course, if it is necessary to decrease one or more batch sizes, the random distribution only selects among the non-zero batch sizes. Note that if the desired per-group batch size is larger than the group size, we do not oversample, and reduce it.

\paragraph{Scaling to large numbers of groups}

In this work, we have considered scenarios where the number of groups $G$ is small, up to 10. $G$ has relatively limited effect on training, as the function $\xi$ is precomputed for all possible batch sizes for each group at the start of training, which is linear in $G$, but not in $T$, the number of training steps. Each model step is effectively constant in $G$, with the same complexity as \DPSGD by performing all operations in parallel across the batch, and the same applies to the group reweighting steps. Of course, as $G$ increases, the sampling behaviour of \acronym becomes more complex, and in extreme cases we would consider an extension similar to Federated Learning or Random Allocation \citep{shenfeld2025privacy} where each group only contributes to a subset of training steps.

\paragraph{Privacy guarantees}

\begin{proof}[Proof of \cref{thm:method_dp}]
  \Cref{alg:method} performs $T$ rounds of model updates and $\lfloor T/k\rfloor$
  updates of the per-group batch sizes.
  We first establish
  the privacy of each of these steps in isolation.

  \begin{lemma} \label{lem:model_update_RDP}
    At each step $t\in\{1,\dots,T\}$, the model update step of \Cref{alg:method} is $(\alpha,\epsilon^s)$-RDP.
  \end{lemma}

  \begin{proof}
    Let $D=(D_g)_{g\in[G]}$ and $D'=(D'_g)_{g\in[G]}$ be two neighboring datasets.
    Let $g$ be the group in which the difference between both sets occur, \ie
    $D_g$ and $D'_g$ are neighboring, and so differ on exactly one point, and all other group datasets are identical.
    The model update step on $D_g$ consists of a \emph{subsampled Gaussian mechanism}
    with sampling rate $\gamma_g=\frac{m_g}{n_g}$, clipping threshold $\xi=\frac{\sigma}{\kappa_\alpha(\gamma_g,\epsilon^s)}$ and noise strength $\sigma^2$.
    By definition of $\kappa_{\alpha}$, the Rényi divergence at order $\alpha$ between the mechanism's output on $D$ and $D'$ is bounded by $\epsilon^s$, meaning that the mechanism fulfills $(\alpha,\epsilon^s)$-RDP, as this property holds for all groups $g$, and so all neighboring datasets.
  \end{proof}

  \begin{lemma}\label{lem:batchsize_update_RDP}
    At each step $t=lk$ for $l\in\{1,\dots\lfloor\frac{T}{k}\rfloor\}$,
    the batch size update step of \Cref{alg:method} is $\epsilon_\alpha(\gamma^\ell,\tau/\zeta)$-RDP.
  \end{lemma}

  \begin{proof}
    The batch size update accesses private data only in the computation
    of the losses (line 13). There, it applies a standard subsampled Gaussian
    mechanism with sampling rate $\gamma^\ell$, clipping threshold $\zeta$
    and noise scale $\tau^2$, resulting in a privacy guarantee of
    $\epsilon_\alpha(\gamma^\ell, \tau/\zeta)$-RDP. The remaining steps preserve
    this privacy level due to RDP's postprocessing property.
  \end{proof}

  \emph{Proof of \Cref{thm:method_dp} -- continued.}
  Apart from its own access to the private data, each
  step of \Cref{alg:method} relies only on values that
  were either specified as public inputs, or that were
  computed previously in a private manner. Consequently,
  we obtain its overall privacy guarantees using the
  adaptive composition property of Rényi-DP~\citep{mironovRenyiDifferentialPrivacy2017}.
  Because the algorithm executes $T$ model update steps, each guaranteed to be $(\alpha,\epsilon^s)$-RDP, and $\lfloor\frac{T}{k}\rfloor$ reweighting
  steps, each of which is $(\alpha,\epsilon_\alpha(\gamma^\ell,\tau/\zeta))$-RDP, the statement
  of \Cref{thm:method_dp} follows.
\end{proof}

\section{Privacy of \cref{alg:zhou}}
\label{app:zhou-priv-proof} 
\begin{proposition}[name=Privacy guarantees for \cref{alg:zhou} and its variants]
  \label{prop:zhou-priv}
  \Cref{alg:zhou} has a privacy guarantee stated in Renyi DP given by:  $(\alpha,
  \epsilon_\text{\zhou})$ given by
  \begin{equation}
    \epsilon_\text{\zhou} =  
    T \cdot \epsilon_\alpha\left( \boldsymbol{\frac{M}{\min_g n_g}}, \frac{\sigma}{\xi}\right)
    +
    \frac{T}{k} \cdot \epsilon_\alpha\left(\gamma^\ell,\frac{\tau}{\zeta}\right),
  \end{equation}

  The proportional variant \zhouprop of \cref{alg:zhou} has RDP guarantees $(\alpha,
  \epsilon_\text{\zhouprop})$ given by
  \begin{equation}
    \epsilon_\text{\zhouprop} =
    T \cdot \epsilon_\alpha\left(\boldsymbol{\frac{M}{N}}, \frac{\sigma}{\xi}\right)
    +
    \frac{T}{k} \cdot \epsilon_\alpha\left(\gamma^\ell, \frac{\tau}{\zeta}\right),
  \end{equation}
  matching the guarantees of standard \DPSGD, where we sample a batch of size $M$ from the whole dataset (size $N$), with of course, an added term for the reweighting of $\lambda$.
\end{proposition}

\begin{proof}

  Consider two neighboring datasets $D=(D_g)_{g\in[G]}$ and $D'=(D'_g)_{g\in[G]}$, and let $g$ be the group in which the difference occurs, so that $D_g$ and $D'_g$ are neighboring, and all other groups' datasets are identical.

  Each model update step applies noise with variance $\sigma^2$ to gradients clipped to stay within a  $\xi$-radius $L_2$ ball, giving a noise multiplier of $\sigma/\xi$.
  In every step, it is possible that group $g$ is selected, but the likelihood thereof is unknown, as the sampling weights $\lambda_g$ could take any value, and also change over time . For a DP proof, we must take the worst case, so we will account for the privacy cost of $g$ being chosen.
  If so, the the subsampling probability from $D_g$ is then $\gamma_g=\frac{M}{n_g}$ for \zhou. As the privacy bound should hold uniformly over groups, the worst-case subsampling probability is therefore $\gamma=\frac{M}{\min_g n_g}$. For \zhouprop, the sampling probability is  $\gamma=\frac{M}{N}$ by cancelling the proportionality terms in the batch size definition.
  Thus, we suffer a privacy loss of $\epsilon_\alpha(\gamma, \sigma/\xi)$ for each of the $T$ steps, regardless of how many times group $g$ actually is sampled.

  For the reweighting of $\lambda_g^{(t)}$, the losses on all datapoints are released evey $k$ steps, using a standard subsampled Gaussian mechanism with sampling rate $\gamma^\ell$, noise multiplier $\tau/\zeta$, giving cost $\epsilon_\alpha(\gamma^\ell,\tau/\zeta)$ for each of these $T/k$ steps (where we assume, for the sake of simplicity, that this evenly divides).

  Summed together, the RDP costs of these steps gives the guarantees above for \zhou, and \zhouprop.

\end{proof}

\subsection{Discussion on \zhou}
\label{app:discuss-zhou}

\Cref{sec:experiments} discusses the fact that \zhou's performance is very low across all datasets. This is largely a result of the mismatch between the algorithm's design and its use on datasets with unequal group sizes.
Imbalanced groups are a natural phenomenon, as \eg it is expected that a small proportion of the population commits bank fraud, even though high accuracy on this group is actually very important. 

We can see the problem in more detail if we turn to the the CelebA dataset, which has four groups, ranging in size from $1387$ to $71629$ images.
For all methods, we used a consistent batch size $b=256$ (and also tried a smaller $b=36$ specifically for \zhou to help address this problem), giving a sampling rate $\gamma_\text{min} = 256/1387=0.186$ for the minority group. On the other hand \DPSGD would use the full dataset size, over $162,000$, as the denominator, which leads to a significantly smaller sampling rate and thus better privacy protection. The noise multiplier $\sigma/\xi$ is automatically calculated, like in Opacus \citep{yousefpourOpacusUserFriendlyDifferential2022}, to give an $\epsilon=1$ guarantee for all groups. The result, $\sigma/\xi =570$, (respectively $\sigma/\xi =214$ for $b=36$), is over $100$ (resp. $40$) times higher than the noise multiplier for DPSGD, $5.08$, or that for our method and \zhouprop, $5.59$. This is extremely high, and has a severely negative effect on the training process.

One approach to fix this is \zhouprop, which shows much better results in experiments (\cref{sec:experiments}), but has the highest sampling variance of all methods (\cref{thm:sampling-variances}). 

\section{Convergence and sampling variance}
\label{app:convergence}
\renewcommand{\truegrad}{\ensuremath{U^{\text{WGO}}}}
\renewcommand{\stochgrad}{\ensuremath{\tilde{U}}}
\newcommand{\genericalgname}{\tt{Stochastic-Noisy-WGO}}
\newcommand{\dpNoiseStdDev}{\ensuremath{\sigma_\text{DP}}}

\newcommand{\samplingStdDev}{\ensuremath{v}}
We will prove convergence guarantees for noisy WGO algorithms, taking into account the sampling behaviour of the algorithms, the DP noise, and the min-max optimisation nature of these algorithms. Private WGO algorithms also include clipping steps, but we focus here on a simplified scenario without clipping, as this significantly complicates convergence analyses.

For the remainder of this section, we introduce the following notation:

Let $\phi(\theta, \lambda) = \sum_g \lambda_g \mathcal{L}(D_g, \theta)$ be the empirical WGO objective, where $\mathcal{L}(S) := \frac{1}{|S|}\sum_{z \in S}\ell(z, \theta)$ is the average loss over the data points in $S$ with respect to the current model parameters $\theta$. For any dataset $S$,  $U_t(S) := \frac{1}{|S|}\sum_{z \in S}\nabla\ell(z, \theta_t)$ is the gradient of the loss on this set at time step $t$. Therefore, the non-stochastic, full gradient of $\phi(\theta,\lambda)$ with respect to $\theta$  is 
\begin{equation}
  \truegrad_t := \sum\nolimits_{g \in [G]} \lambda_g U(D_g).
\end{equation}.

Consider the generic algorithm given in \cref{alg:generic-noisy-wgo}, which is parameterised by a sampling function $\stochgrad(D, \lambda)$, which will give a noisy estimate of the true gradient, including the weighting by $\lambda$. We will later assume a bound on the variance of this function, defined as 
\begin{equation}
  \var(\stochgrad) = \ex \n\truegrad - \stochgrad\n_2^2,
\end{equation}
which will be required to hold at all steps $t$ of optimisation (\ie for all $\lambda,\theta$), so we elide these parameters.
Define $\bar{\lambda} = \frac1T \sum_t \lambda_t$ and $\bar{\theta} = \frac1T \sum_t \theta_t$ to be the average of the group weightings and model parameters across all $T$ iterations.
 \begin{algorithm}[ht]
    \centering
    \caption{\genericalgname}  \label{alg:generic-noisy-wgo}
    \begin{algorithmic}[1]
      \REQUIRE
      sampling function $\stochgrad$,
      learning rate $\eta$,

      \smallskip
      \STATE initialize model parameters $\theta_1\in \Theta$
      \STATE initialize per-group weights $\lambda_g=\frac{1}{G}$ for $g\in[G]$
      \FOR{$t = 1,\dots, T$}
     
      \STATE sample a batch $B \leftarrow \stochgrad(D, \lambda)$  
      \STATE
      $\displaystyle \theta_{t+1} \leftarrow \mathrm{Proj}\left(\theta_{t} -
      \frac{\eta}{|B|}  \left(  \sum\limits_{z\in B}
      \nabla\ell(z,\theta_t) + \N(0,  \dpNoiseStdDev^2 I)  \right)  \right) $ 
      \STATE $\lambda \leftarrow \texttt{DP-Group-Reweight}(\lambda)$ 
      \ENDFOR %
    \end{algorithmic}
  \end{algorithm}

\subsection{Convergence of noisy WGO algorithms}

\begin{theorem}[name=Convergence of \genericalgname,label=thm:genericConvergenceThm]
  \Cref{alg:generic-noisy-wgo}, when applied to a convex $L$-smooth loss function, which is bounded in $[0,1]$, on a weight space bounded by $D$ in $\ell_2$ norm, enjoys the following convergence guarantee, parameterised by the standard deviations $\dpNoiseStdDev$ and $\samplingStdDev$, of the DP noise and the gradient sampling, respectively, and the gradient norm $G_0$ at the optimal solution: 

  \begin{equation}
    \begin{aligned}
       \ex[\max_g \mathcal{L}(D_g, \bar{\theta})] - \min_\theta \max_g \mathcal{L}(D_g, \theta) \leq  &O\left(\frac{D(L \cdot D + G_0 + \dpNoiseStdDev + \samplingStdDev)}{\sqrt{T}}\right) \\
        &+ O\left(\sqrt{\frac{ \tau^2  \log (G)\log(n^2 T)}{T}}\right).
    \end{aligned}
  \end{equation}

  when the learning rates are set by  $\eta_\theta =D/(\sqrt{T} \cdot (L \cdot D + G_0 + \dpNoiseStdDev + v))$ and $\eta_\ell = \sqrt{(\log G)/T A^2}$ for $A = 1 + \sqrt{2 \tau^2 \log (GT^2)}$.

\end{theorem}
\begin{proof}
  Initially, we decompose the optimality gap into two terms, one based on the (non-)optimality of the model parameters, and the second on the group weights $\lambda$.
  \begin{align}
    & \max_g \mathcal{L}(D_g, \bar{\theta}) - \min_\theta \max_g \mathcal{L}(D_g, \theta) \\
    &=
    \max_\lambda \phi(\bar{\theta}, \lambda)
    - \min_\theta \max_\lambda \phi(\theta, \lambda) \\
    &\leq
    \max_\lambda   \phi(\avgtheta, \lambda)
    - \min_\theta \phi(\theta, \avglambda)\text{ (by non-optimality of } \avglambda,\avgtheta) \\
    &\leq \max_{\lambda, \theta} \{
      \frac1T
      \sum_{t=1}^{T} (\phi(\theta_t, \lambda) - \phi(\theta, \lambda_t))
    \} \text{ (by convexity)} \\
    &\leq \max_{\lambda, \theta} \{
      \frac1T
      \sum_{t=1}^{T} (\phi(\theta_t, \lambda) - \phi(\theta_t, \lambda_t) + \phi(\theta_t, \lambda_t) - \phi(\theta, \lambda_t))
    \}   \\
    &= \frac1T \underbrace{\left[
        \sum_t \phi(\theta_t, \lambda_t) - \min_\theta \sum_t \phi(\theta, \lambda_t)
    \right]}_{(A)} + \frac1T \underbrace{\left[
        \max_\lambda \sum_t \phi(\theta_t, \lambda) - \sum_t \phi(\theta_t, \lambda_t)
    \right]}_{(B)}.
  \end{align}

  Beginning with term A,
  define $\opttheta(\avglambda) = \arg\min_\theta \phi(\theta, \avglambda)$ to be the optimal model parameters for the average weighting $\avglambda$. Note that $\{\lambda_t\}_{t =1}^T$ is a trajectory of dependent random variables, and $\opttheta(\avglambda)$ depends on this.

  Given the sequence of stochastic updates $\{\stochgrad_t\}$ used by the algorithm, the update step satisfies:
  \begin{equation}
    \langle \stochgrad_t, \theta_t - \opttheta(\avglambda) \rangle
    \leq \frac{1}{2\eta_\theta} (\n\theta_t - \opttheta(\avglambda) \n_2^2 - \n \theta_{t+1} - \opttheta(\avglambda) \n_2^2) + \frac{\eta_\theta}{2} \n \stochgrad_t \n_2^2,
  \end{equation}

  Using the fact that the optimisation space has diameter bounded by $D$, this bound telescopes to give
  \begin{equation}
    \sum_{t=1}^T \langle \stochgrad_t, \theta_t - \opttheta(\avglambda) \rangle
    \leq \frac{D^2}{2\eta_\theta} + \frac{\eta_\theta}{2} \sum_{t=1}^T \n \stochgrad_t \n_2^2.
  \end{equation}

  The convexity of $\phi$ relates the optimality gap to the inner products above, but applied to the non-stochastic gradients $\truegrad_t$:
  \begin{align}
    \phi(\theta_t, \lambda_t) -  \phi(\opttheta(\avglambda), \lambda_t)
    \leq \langle \truegrad_t, \theta_t - \opttheta(\avglambda) \rangle
  \end{align}

  Taking into account the discrepancy between the stochastic and non-stochastic gradients, we have the following statement, in expectation, as the noisy gradient is an unbiased estimator only when we condition on the past.

  \begin{align}
    \sum_t \phi(\theta_t, \lambda_t) -  \phi(\opttheta(\avglambda), \lambda_t)
    &\leq \frac{D^2}{2\eta_\theta} + \frac{\eta_\theta}{2} \sum_{t=1}^T \n \truegrad_t \n_2^2+ \frac{\eta_\theta}{2} \sum_{t=1}^T \n \stochgrad_t - \truegrad_t\n_2^2 \\
    &+ \sum_{t=1}^T \langle \truegrad_t - \stochgrad_t, \theta_t - \opttheta(\avglambda) \rangle.
  \end{align} \label{eq:bound-on-term-A}

  Working through each of these terms, smoothness bounds the norm of the true gradients $\n \truegrad_t \n_2$ by $L \cdot D + G_0$, where $L$ is the smoothness constant, and $G_0$ is the norm of the gradient at the optimal solution. Next, the sampling variance term is bounded in expectation by assumption, and we are left with the inner product term.

  This term arises as the optimal weights $\avglambda$ depend on the trajectory of the algorithm, and thus the stochastic gradients are not unbiased estimates of the true gradients with respect to $\avglambda$. 
  It can be decomposed further, into

  \begin{equation}
    \langle \truegrad_t - \stochgrad_t, \theta_t - \opttheta(\avglambda) \rangle
    = \langle \truegrad_t - \stochgrad_t, \theta_t  \rangle
    -\langle \truegrad_t - \stochgrad_t, \opttheta(\avglambda) \rangle.
  \end{equation}

  The expectation of this first term is $0$, as the stochastic gradients are unbiased estimates of the true gradients with respect to $\theta_t$. However, the second term is not measurable from the past, and thus we need to bound it properly. Using Cauchy-Schwarz and the fact that $\opttheta(\avglambda)$ is bounded in norm by $D$, we have that
  
  \begin{equation}
    \ex[-\langle \truegrad_t - \stochgrad_t, \opttheta(\avglambda) \rangle]
    \leq D \cdot \ex[\n\truegrad_t - \stochgrad_t\n_2].
  \end{equation} \label{eq:theta-theta-inner-product-bound}

  Combining the bounds on all of the terms in \cref{eq:bound-on-term-A}, and using Jensen's inequality to relate \cref{eq:theta-theta-inner-product-bound} to the sampling variance, we have that:
  \begin{equation}
    \begin{aligned}
      \ex\left[\sum_t \phi(\theta_t, \avglambda) -  \phi(\opttheta(\avglambda), \avglambda) \right]
      \leq
      &\frac{D^2}{2\eta_\theta}
      + \frac{\eta_\theta}{2} \left(
        \sum_{t=1}^T \n \truegrad_t \n_2^2 +
      \sum_{t=1}^T \n \stochgrad_t - \truegrad_t\n_2^2\right) \\
      &+ D \cdot \sum_{t=1}^T \ex[\n\truegrad_t - \stochgrad_t\n_2].\\
      &\leq \frac{D^2}{2\eta_\theta}
      + \frac{T\eta_\theta}{2} (L \cdot D + G_0)^2
      + \frac{T\eta_\theta}{2} (\dpNoiseStdDev^2 + v^2)
      + D\sqrt{T} v.
    \end{aligned}
  \end{equation}

  Picking $\eta_\theta$ to be $\frac{D}{\sqrt{T} \cdot (L \cdot D + G_0 + \dpNoiseStdDev + v)}$ gives us a bound of $O\left(\frac{D(L \cdot D + G_0 + \dpNoiseStdDev + v)}{\sqrt{T}}\right)$ on term (A).

  For term (B), we follow the proof of \cite{zhouDifferentiallyPrivateWorstgroup2024}, which uses Hedge's algorithm.

  First, define a constant $A = 1 + \sqrt{2 \tau^2 \log (GT^2)}$, and noised, transformed losses $L'_{t,g} = U - L_t (D_g) + \z_g$, where $\z_g$ is Gaussian noise with variance $\tau^2$. The loss inputs to the Hedge algorithm should be appropriately bounded, and we know that by the choice of $A$, $\prob (\forall t \n L'_t \n_\infty \leq A) \geq 1-\frac{1}{T}$, so we can condition on this event and use $A$ as the bound on the losses.

  Conditioned on this event $E_A$, the regret bound of the Hedge algorithm \citep{slivkinsIntroductionMultiArmedBandits2024} has expectation
  \begin{equation}
    \ex[B | E_A] = O\left(A \sqrt{\frac{\log G}{T}}\right),
  \end{equation}

  given the learning rate is set to $\sqrt{\frac{\log G}{T A^2}}$.

  Applying the definition of $A$, and using the worst-case bound of $1$ on the losses when $E_A$ does not hold, we have that
  \begin{equation}
    \ex[B] \leq O\left(\sqrt{\frac{ \tau^2 \log (G)\log(G^2 T)}{T}}\right) + \frac{1}{T},
  \end{equation}

  where the last term disappears faster than the first.

\end{proof}

\subsection{Sampling variance of WGO algorithms}
\label{app:sampling_variance_proof}

For any dataset $S$ and $m\leq |S|$, let $\sample(m,S)$
be an operator that samples a batch of size $m$ from $S$
without replacement.
\acronym's stochastic update can be written as
\begin{align} %
  U^{\text{\acronym}} &= U \Big({\textstyle\bigcup\nolimits_{g \in [G]}} \sample(m_g,D_g) \Big).%
  \intertext{For \zhou, the updates are }
  U^{\text{\zhou}} &= U(\sample(M,D_g)) \text{ for $g\sim\text{Cat}(\lambda)$}
  \intertext{For \zhouprop the
  update is}
  U^{\text{\zhouprop}} &= U(\sample(m_g,D_g)) \text{ for $g\sim\text{Cat}(\lambda)$}.
  \intertext{For \lossreweighting the
  update is}
  U^{\text{\lossreweightingShort}} &= U(\sample(M,D_\text{\lossreweightingShort})) \text{ for } D_\text{\lossreweightingShort} = \left\{\frac{\lambda_g N}{n_g} z : z \in D_g, g \in [G]\right\}  .
\end{align}
It is easy to check that %
all of these updates are unbiased estimates of the WGO update, \ie
\begin{equation}
  \ex[U^{\text{\acronym}}]
  = \ex[U^{\text{\zhou}}]
  = \ex[U^{\text{\zhouprop}}]
  = \ex[U^{\text{\lossreweightingShort}}]
  = U^{\text{WGO}}.
\end{equation}

Repeating the theorem for convenience, we have:

\getkeytheorem{samplingVarianceThm}

\begin{proof}[Proof of \cref{thm:sampling-variances}]

  We start with the more complicated derivation, for \zhou, beginning with the covariance matrix of  $U^\text{\zhou}$. Given group $g$ was selected, define the update that would be performed by $U_{g}^\text{\zhou} = U(\sample(M, D_g))$.

  The following holds by the law of total covariance, applied to the fact that this sampling scheme is defined by a mixture distribution, with weights $\{\lambda_g\}_{g \in [G]}$:

  \newcommand{\Uwgo}{\ensuremath{U^{\text{WGO}}}}

  \begin{equation}
    \begin{aligned}
      \cov(U^\text{\zhou})
      &= \sum_{g} \lambda_g \cov(U_{g}^\text{\zhou})  \\
      &+ \sum_g \lambda_g
      (U(D_g) - \Uwgo)
      (U(D_g) - \Uwgo)^\top
      .
    \end{aligned}
  \end{equation}

  Next, we turn to the covariance of the average update on the sampled batch, $\cov(U_{g}^\text{\zhou})$, where the samples are drawn without replacement. This is covered by a simple lemma:

  \begin{lemma}
    \label{lem:cov-wor-sampling}
    Consider a batch of size $m$, represented by indices $\{I_j\}_{j \in [m]}$, which is sampled without replacement from a population size $\{v_i\}_{i \in [n]}$.

    The population covariance matrix is given by $\cov(v_I) = \ex[(v_{I} - \ex[v_{I}])(v_{I} - \ex[v_{I}])^\top]$ for $I \sim \mathrm{Unif}[n]$, and the cross-covariance matrix is $K_{IJ}= \ex[(v_{I} - \ex[v_{I}])(v_{J} - \ex[v_{J}])^\top]$ for $I\neq J$ drawn uniformly without replacement from $[n]$. %

    The covariance matrix of the batch sum $\cov(\sum_{i \in [m]} v_{I_i})$ is given by
      \begin{equation}
        \cov\left(\sum_{i \in [m]} v_{I_i}\right)
        =
        m \cdot\frac{n-m}{n-1} \cdot \cov(v_I).
      \end{equation}

    \end{lemma}

    \begin{proof}

      We start by expanding the batch sum's covariance over the sum, giving us

      \begin{equation}
        \cov\left(\sum_{i \in [m]} v_{I_i}\right)
        =
        m \cdot \cov(v_I) + (m^2 -m) \cdot K_{IJ}.
      \end{equation}

      $K_{IJ}$ is eliminated by considering a batch which is the same size $m=n$ as the whole dataset, which obviously has no covariance as the only randomness is over the sampling scheme. Solving $\cov\left(\sum_{i \in [n]} v_{I_i}\right) = \mathbf{0}$ leads to $K_{IJ} = \frac{-1}{n-1} \cov(v_I)$ and the result in the lemma statement.
    \end{proof}

    Thus, $\cov(U_{g}^\text{\zhou}) = \frac{1}{M} \frac{n_g - M}{n_g -1} \cov_g$, where $\cov_g$ is shorthand for $\cov_{z \sim D_g}(\nabla_{\theta}(\ell(z,\theta)))$. Together, this gives us

    \begin{equation}
      \label{eq:zhou-cov-fully-expanded}
      \begin{aligned}
        \cov(U^\text{\zhou})
        &= \sum_{g} \lambda_g  \frac{1}{M} \frac{n_g - M}{n_g -1} \cov_g  \\
        &+ \sum_g \lambda_g
        (U(D_g) - \Uwgo)
        (U(D_g) - \Uwgo)^\top
        .
      \end{aligned}
    \end{equation}

    For \zhou, the final step to determine $\var(U^{\text{\zhou}})$ is to take the trace of both sides of \cref{eq:zhou-cov-fully-expanded}, as the trace is linear and cyclic:

    \begin{equation}
      \label{eq:zhou-var-post-trace}
      \begin{aligned}
        \var(U^{\text{\zhou}})
        &= \sum_{g} \lambda_g  \frac{1}{M} \frac{n_g - M}{n_g -1} \var_g  \\
        &+ \sum_g \lambda_g
        \left\| U(D_g) - \Uwgo \right\|_2^2
        .
      \end{aligned}
    \end{equation}

    Note that the variance reduction given by using batch sizes larger than $1$ does not apply to the second term of this variance, which depends solely on the distance between the per-group update and the overall update.

    The same proof holds for \zhouprop, except for the replacement of $M$ with $m_g$ throughout.

    For our method, we rewrite the variance by absorbing the expected update $\Uwgo$ into the sums over the batch parts, which is possible as the sub-batches $B^\text{\acronym}_g := \sample(m_g, D_g)$ have sizes $M\cdot \lambda_g$:
    \begin{align}
      \var(U^\text{\acronym})
      &=
      \ex \left\|U ^\text{\acronym} - \Uwgo\right\|^2 \\
      &=
      \ex \left\n \left(
        \frac1M \sum_{g \in [G]}
        \sum_{z \in B^\text{\acronym}_g}
      \nabla_{\theta} \ell(z, \theta) \right)
      -  \sum_{g \in [G]} \lambda_g U(D_g)
      \right\n_2^2 \\
      &=
      \ex \left\n
      \frac1M \sum_{g \in [G]}
      \sum_{z \in B^\text{\acronym}_g}
      \left(\nabla_{\theta} \ell(z, \theta) -U(D_g)\right)
      \right\n_2^2. \\
      \intertext{
        By the definition of $U(D_g)$, the per-group sums $\sum_{z \in B^\text{\acronym}_g}  \left(\nabla_{\theta} \ell(z, \theta) -U(D_g)\right)$ have zero expectation, and are independent as they are sampled from different groups:
      }
      &=
      \frac1{M^2}
      \sum_{g \in [G]}
      \ex \left\n
      \sum_{z \in B^\text{\acronym}_g}
      \nabla_{\theta} \ell(z, \theta) - U(D_g)
      \right\n_2^2 \\
      \intertext{
        Applying \cref{lem:cov-wor-sampling}, the relationship $\var(X) = \ex \|X\|_2^2=\operatorname{trace}(\cov(X))$ and the linearity of $\operatorname{trace}$, this simplifies to
      }
      &=
      \frac1{M^2}
      \sum_{g \in [G]}
      m_g \cdot \frac{n_g - m_g}{n_g - 1}
      \var_g \\
      &=
      \frac1{M}
      \sum_{g \in [G]}
      \lambda_g \cdot \frac{n_g - m_g}{n_g - 1}
      \var_g \\
    \end{align}

    Finally, for \lossreweighting, the estimator is the sample mean of a batch of size $M$ drawn without replacement from the modified population $V = \left\{ v_z \right\}_{z \in D}$, where we define the per-sample reweighted gradient as $v_z = \frac{N \lambda_{g(z)}}{n_{g(z)}} \nabla_\theta \ell(z, \theta)$.

Applying \cref{lem:cov-wor-sampling} to the entire dataset $D$, the covariance matrix of the batch sum is $M \frac{N-M}{N-1} \cov(v_z)$. Because our update $U^{\text{Reweight}}$ is the average over the batch, we scale this variance by $\frac{1}{M^2}$. Taking the trace yields:
$$
\var(U^{\text{Reweight}}) = \frac{1}{M} \frac{N-M}{N-1} \var_{z \sim \mathrm{Unif}(D)}(v_z).
$$

To decompose the population variance $\var(v_z)$, we apply the law of total variance conditioned on the group identity $g(z)$. A uniformly drawn sample from $D$ belongs to group $g$ with probability $\frac{n_g}{N}$. 
Next, we calculate the expected intra-group variance, as:
$$
\ex_g [\var(v_z \mid g)] = \sum_{g \in [G]} \frac{n_g}{N} \left( \frac{N \lambda_g}{n_g} \right)^2  \var_g = \sum_{g \in [G]} N \frac{\lambda_g^2}{n_g} \var_g.
$$

Next, we calculate the variance of the intra-group expectations. Conditioned on $g$, the expected vector is $\ex[v_z \mid g] = \frac{N \lambda_{g(z)}}{n_{g(z)}} U(D_g)$. Since the overall expectation over the full population is $U^{\text{WGO}}$, we have:
$$
\var_g (\ex[v_z \mid g]) = \sum_{g \in [G]} \frac{n_g}{N} \left\| \frac{N \lambda_g}{n_g} U(D_g) - U^{\text{WGO}} \right\|_2^2.
$$

Summing the two terms yields the population variance $\var(v_z)$, which we substitute back to arrive at the final sampling variance:
$$
\var(U^{\text{Reweight}}) = \frac{1}{M}\frac{N - M}{N - 1} \left( \sum_{g \in [G]} N \frac{\lambda_g^2}{n_g} \var_g + \sum_{g \in [G]} \frac{n_g}{N} \left\| \frac{N \lambda_g}{n_g} U(D_g) - U^{\text{WGO}} \right\|_2^2 \right).
$$

  \end{proof}

\section{Experiment Details}
\label{app:experiments}
\subsection{Datasets}

We use three standard benchmarks for fair classification:
\begin{itemize}[nosep]
  \item \emph{Unbalanced MNIST}~\citep{esipovaDisparateImpactDifferential2022},
    a modified version of the classical MNIST~\citep{lecun2010mnist}, in which class 8 is downsampled
    by a factor of approximately 10. Each individual digit constitutes a group, \ie groups are
    aligned with the classification targets.
    We train a two-layer CNN as in~\citep{esipovaDisparateImpactDifferential2022}.
    \smallskip
  \item \emph{CelebA}~\citep{liuDeepLearningFace2015}, a large-scale face classification dataset.
    The learning task is to predict the \emph{male/female} attribute.
    Here, the group are more fine-grained than the classification targets, namely
    \emph{blond male} (a clear minority group), \emph{blond female}, \emph{non-blond male},
    and \emph{non-blond female}. %
    As a model, we finetune a pretrained ResNet50~\citep{heDeepResidualLearning2016}.
    \smallskip
  \item \emph{Bank Account Fraud}~\citep{jesusTurningTablesBiased2022}, a set of large-scale tabular datasets about predicting fraudulent behavior when opening bank accounts.
    The four groups are the pairwise combinations of \emph{fraudulent}/\emph{not fraudulent} with \emph{age under/over 50}. We train a small MLP with 2 hidden layers of size 256 and ReLU nonlinearities.
\end{itemize}
\subsection{Models}

For the Unbalanced MNIST dataset, we train the same small CNN as
\citet{esipovaDisparateImpactDifferential2022}, which has two
hidden convolutional layers with 32 and 16 channels respectively, and
kernel size 3 followed by a fully connected layer to the 10 output
classes. The model uses Tanh activations.

The ResNet50 for CelebA used the Huggingface weights, which were pretrained on the ImageNet50 dataset. Note that the
ResNet50 model uses GroupNorm \cite{wuGroupNormalization2020} with 32
groups instead of BatchNorm to ensure compatibility with DPSGD,
using Opacus \cite{yousefpourOpacusUserFriendlyDifferential2022}'s
automatic conversion.

\subsection{Computing Resources}

Our experiments on the Unbalanced MNIST and Bank Fraud datasets were run on CPU only, with each run taking at most 1 hour / 30 minutes respectively. Our experiments on CelebA were run on single gpus, Nvidia A10 or better, with each run taking at most 12 hours. All of our experiments were implemented in PyTorch \citep{pytorch}.

\subsection{Uncertainty and random seeds}

All tables and figures include uncertainty estimation, corresponding to one standard deviation over the random seeds used for each hyperparameter setup.
For \cref{tab:headline_combined}, 5 seeds were used, whereas we used 3 for \cref{fig:wga_eps} due to the increased number of runs necessary to generate this figure. \Cref{fig:sampling_variance_for_algs} was plotted using 10 seeds for each setup, to reduce noise in the visualization. \cref{fig:wga_variance_mnist,fig:wga_variance_celeba,fig:wga_variance_bank} are all plotted from the same runs as \cref{tab:headline_combined}.

\subsection{Choosing Hyperparameters}

Experimental baselines were not available for all of these datasets and methods under privacy, meaning that prior information on what parameters represent a reasonable learning scenario was lacking.
Thus, we designed a small sweep of plausible hyperparameter combinations (\eg learning, rate, momentum,...), and tested each method on all combinations, and selected the best combination for each method, measured by the worst group accuracy on the validation set (which was not used at any other point in our experiments).

Note that the hyperparameter sweep for \zhou  also included an extra, smaller batch size for the CelebA dataset, to include some runs with somewhat smaller noise multipliers. As the number of steps per epoch is calculated by taking $\text{dataset size}/\text{batch size}$, this results in more steps being taken overall, but with the same final privacy guarantee.
This hyperparameter sweep is not accounted for in our privacy guarantees, as it would incur the same privacy cost for all methods, and was not performed privately, as this would reduce the accuracy of our comparisons if noise was introduced into selection. This means that our results leak the privacy of the validation set, which would be unacceptable if applied to a real learning problem, but is permissible here to be able to demonstrate all of the methods fairly.

For DP methods, $\delta$ was set to $\frac{1}{2N}$, where $N$ is the
total number of datapoints in the dataset.

Our experiment hyperparameters are given in the following table, for
\cref{tab:headline_combined}. \Cref{fig:sampling_variance_for_algs} was run using the best hyperparameter combinations found in the runs for \cref{tab:headline_combined}, and the same was done for \cref{fig:wga_eps} at   $\epsilon$ values in $[0.1, 0.5, 1.0, 1.33, 1.7, 2.12, 5]$.

\subsection{Implementation of WGO Algorithms}

In \cref{app:hyp-table}, we detail the hyperparameters used for each dataset and method.

To implement all of the algorithms, we set a target $(\epsilon,\delta)$-DP guarantee, and all other parameters, and searched over possible noise multipliers $\sigma/\xi$, (and $\tau/\zeta$) to achieve this guarantee, using \cref{thm:method_dp} and \cref{prop:zhou-priv} as well as conversion from RDP (\cref{prop:conv-rdp-eps}) over a set of reasonable $\alpha$ orders, taken from the privacy accountant library.

For the WGO related parameters, we briefly discuss their implementation:

\begin{itemize}
  \item \emph{WGO NM scaling}: the ratio between $\tau/\zeta$ and $\sigma/\xi$
  \item \emph{WGO noise multiplier}: the resulting noise multiplier for the loss releases in the WGO step
  \item \emph{WGO update frequency}: how frequently, in \# of epochs, the group weights should be updated.
  \item \emph{\# of WGO updates}: the resulting number of WGO updates performed.
  \item \emph{WGO loss clipping}: $\zeta$, used to clip the loss values.
  \item \emph{WGO learning rate}: $\eta$, used in the exponential reweighting.
  \item \emph{loss sampling rate}: set to $1$ for all datasets.
\end{itemize}

\subsection{Hyperparameter combinations}
\label{app:hyp-table}
\renewcommand{\arraystretch}{1.5}
\begin{longtable}{p{2cm} p{4cm} p{4cm} p{4cm}}
\toprule
Configuration & Unbalanced MNIST & CelebA & Bank Fraud \\
\midrule
\endfirsthead
\toprule
Configuration & Unbalanced MNIST & CelebA & Bank Fraud \\
\midrule
\endhead
\midrule
\multicolumn{4}{r}{Continued on next page} \\
\midrule
\endfoot
\bottomrule
\endlastfoot
Number of seeds & $\num{5}$ & $\num{5}$ & $\num{5}$ \\
Model & mnist\_cnn & resnet50 & tabular\_small \\
Accuracy metric & argmax & argmax & argmax \\
Loss & cross\_entropy & cross\_entropy & cross\_entropy \\
$\epsilon$ & $\num{1.00}$ & $\num{1.00}$ & $\num{1.00}$ \\
DP clipping threshold & $\num{1.00}$ & $\num{0.50}$ & $\num{1.00}$ \\
Optimizer & SGD & SGD & SGD \\
Learning rate & $\num{1.00e-02}$, $\num{1.00e-03}$ & $\num{1.00e-02}$, $\num{1.00e-03}$, $\num{1.00e-04}$, $\num{1.00e-05}$ & $\num{1.00e-02}$, $\num{1.00e-03}$, $\num{5.00e-02}$ \\
Momentum & $0$, $\num{0.50}$ & $\num{0.50}$, $\num{0.90}$ & $0$, $\num{0.50}$ \\
Weight decay & $0$ & $0$ & $\num{1.00e-04}$ \\
Nesterov momentum & \textit{False} & \textit{False} & \textit{False} \\
Pretrained weights & \textit{False} & \textit{True} & \textit{False} \\
Batch size & $\num{128}$, $\num{256}$, $\num{512}$ & $\num{256}$: all methods;\newline $\num{36}$: \zhou & $\num{1000}$, $\num{2000}$, $\num{4000}$ \\
\# Epochs & $\num{60}$ & $\num{50}$ & $\num{25}$ \\
Noise multiplier (calc.) & \zhouprop: $\num{11.17}$, $\num{14.53}$, $\num{9.06}$;\newline \acronym (ours): $\num{11.17}$, $\num{14.53}$, $\num{9.06}$;\newline \lossreweightingShort: $\num{11.17}$, $\num{13.59}$, $\num{9.06}$;\newline \zhou: $\num{620.00}$, $\num{670.00}$, $\num{870.00}$;\newline DPSGD-F: $\num{13.12}$, $\num{6.64}$, $\num{9.30}$;\newline DPSGD: $\num{13.05}$, $\num{6.60}$, $\num{9.22}$ & $\num{213.75}$: \zhou;\newline $\num{5.08}$: DPSGD, DPSGD-F;\newline $\num{5.59}$: all except DPSGD, DPSGD-F, \zhou;\newline $\num{570.00}$: \zhou & \acronym (ours): $\num{4.22}$, $\num{5.55}$, $\num{7.54}$;\newline \zhouprop: $\num{4.22}$, $\num{5.55}$, $\num{7.54}$;\newline \zhou: $\num{600.00}$, $\num{850.00}$;\newline DPSGD: $\num{3.83}$, $\num{5.27}$, $\num{7.34}$;\newline \lossreweightingShort: $\num{4.22}$, $\num{5.55}$, $\num{850.00}$;\newline DPSGD-F: $\num{3.87}$, $\num{5.27}$, $\num{7.34}$ \\
$\delta$ & $\num{1.02e-05}$ & $\num{3.07e-06}$ & $\num{7.40e-07}$ \\
\# of WGO updates & $\num{60.00}$ & $\num{50.00}$ & $\num{25.00}$ \\
\# steps, total & $\num{11580.00}$, $\num{23100.00}$, $\num{5820.00}$ & $\num{226100.00}$: \zhou;\newline $\num{31800.00}$: all methods & $\num{16900.00}$, $\num{4225.00}$, $\num{8450.00}$ \\
WGO NM scaling & $\num{10.00}$ & $\num{25.00}$ & $\num{25.00}$ \\
WGO learning rate & $\num{0.10}$, $\num{0.50}$, $\num{1.00e-02}$, $\num{1.00}$, $\num{5.00e-02}$, $\num{5.00e-03}$ & $\num{0.10}$ & $\num{0.10}$, $\num{0.50}$, $\num{1.00e-02}$, $\num{1.00}$, $\num{5.00e-02}$, $\num{5.00e-03}$ \\
WGO loss clipping & $\num{1.00}$ & $\num{1.00}$ & $\num{10.00}$ \\
WGO update frequency & $\num{1.00}$ & $\num{1.00}$ & $\num{1.00}$ \\
\end{longtable}

\clearpage

\myparagraph{Experimental Details for Variance Evaluation}
To visualize the sampling variance during training (\Cref{fig:sampling_variance_for_algs}), we evaluated the theoretical quantities from \Cref{thm:sampling-variances} by drawing fresh samples from the dataset every 10th gradient update. For the Unbalanced MNIST experiments, \acronym and \zhouprop were trained with strict privacy guarantees of $(\epsilon,\delta)=(2.12, 1.02\cdot 10^{-5})$. However, to ensure a comparable DP noise multiplier and achieve a reasonable baseline performance, \zhou was trained with a relaxed parameter of $\epsilon=14.14$. The validation \WGA evaluations (\Cref{fig:wga_variance_mnist}) also include the \lossreweighting algorithm for completeness.

\section{Further Experimental Results}
\label{sec:extra-experiments}

\begin{figure}
  \centering
  \includegraphics{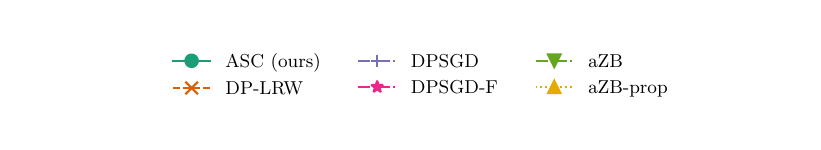}
  \includegraphics{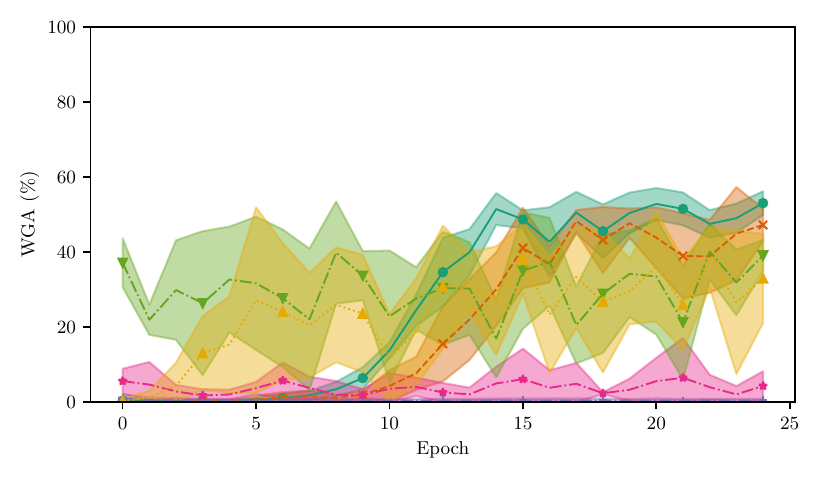}
  \caption{Validation set WGA during training on the Bank Fraud dataset, $(\epsilon=1,\delta=\frac{1}{2n})$.}
  \label{fig:wga_variance_bank}
\end{figure}

In addition to \cref{fig:wga_variance_mnist}, which plots the Worst-case Group Accuracy (\WGA) over the course of training on the Unbalanced MNIST dataset, we also present the same results on the Bank Fraud \cref{fig:wga_variance_bank} and CelebA \cref{fig:wga_variance_celeba} datasets. The plots are generated using the same training runs as \cref{tab:headline_combined}, and the shaded region shows $\pm1$ standard deviation around the mean.

\begin{figure}
  \centering
  \includegraphics{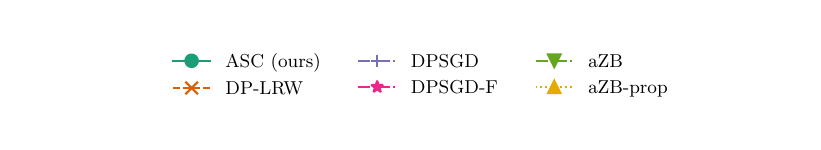}
  \includegraphics{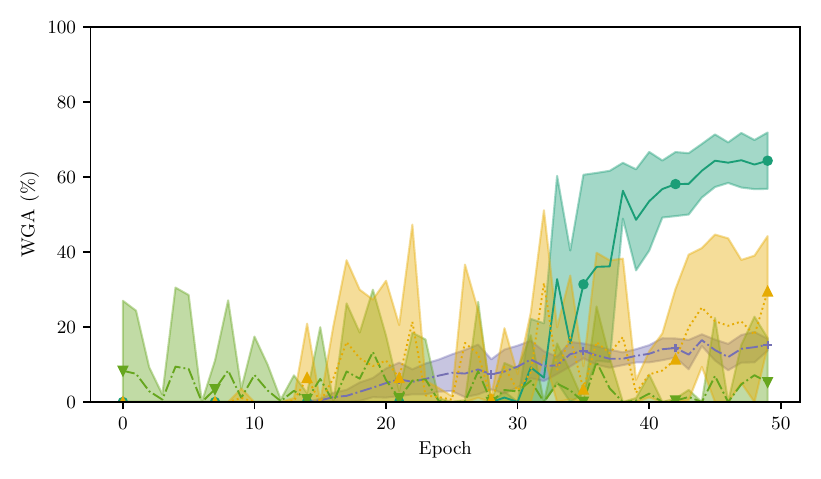}
  \caption{Validation set WGA during training on the CelebA dataset, $(\epsilon=1,\delta=\frac{1}{2n})$.}
  \label{fig:wga_variance_celeba}
\end{figure}

\begin{figure*}
  \centering
  \includegraphics{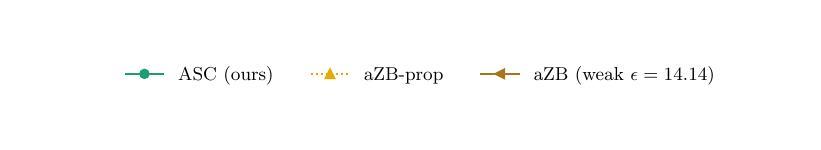}
  \vspace{-1.5em}

  \includegraphics{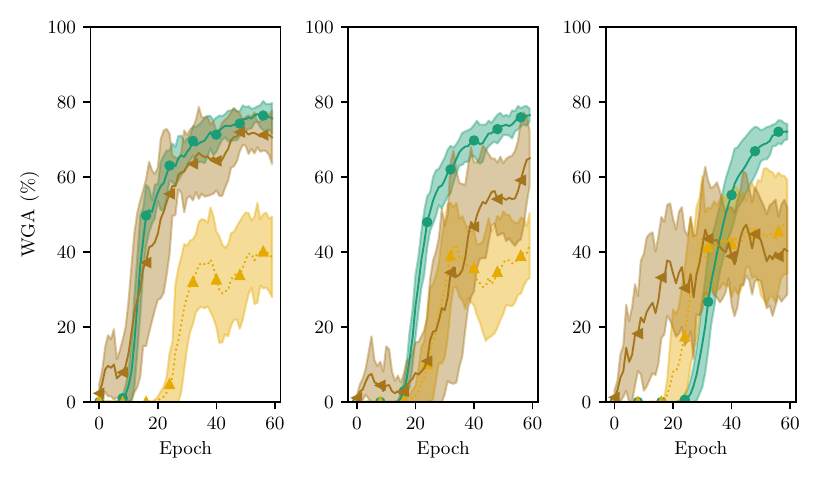}
  \vspace{-2em}

  \begin{subcaptiongroup}
    \centering
    \parbox[t]{.1\textwidth}{%
      \centering  \phantom{spacer}%
    }%
    \parbox[t]{.3\textwidth}{%
    \centering \caption{Batch size $M=128$}}%
    \parbox[t]{.275\textwidth}{%
    \centering \caption{Batch size $M=256$}}%
    \parbox[t]{.3\textwidth}{%
    \centering \caption{Batch size $M=474$}}%
  \end{subcaptiongroup}
  \caption{\WGA on validation set across training, for the same runs as \cref{fig:sampling_variance_for_algs}. Note that \zhou's algorithm was run with significantly higher $\epsilon$ than the others, allowing it to achieve stronger performance, and so reasonable convergence behaviour, than otherwise seen in \cref{tab:headline_combined}.
  }\label{fig:wga_variance_for_sampling_variance_runs}
\end{figure*}

\clearpage
\section*{NeurIPS Paper Checklist}

\begin{enumerate}

\item {\bf Claims}
    \item[] Question: Do the main claims made in the abstract and introduction accurately reflect the paper's contributions and scope?
    \item[] Answer: \answerYes{} %
    \item[] Justification: The main claims made in the abstract and introduction accurately reflect the
paper’s contributions and scope.

\item {\bf Limitations}
    \item[] Question: Does the paper discuss the limitations of the work performed by the authors?
    \item[] Answer: \answerYes{} %
    \item[] Justification: We have discussed the limitations of our work in the conclusion.

\item {\bf Theory assumptions and proofs}
    \item[] Question: For each theoretical result, does the paper provide the full set of assumptions and a complete (and correct) proof?
    \item[] Answer: \answerYes{} %
    \item[] Justification: The full proofs and assumptions are given in the appendices for all results.

    \item {\bf Experimental result reproducibility}
    \item[] Question: Does the paper fully disclose all the information needed to reproduce the main experimental results of the paper to the extent that it affects the main claims and/or conclusions of the paper (regardless of whether the code and data are provided or not)?
    \item[] Answer: \answerYes{} %
    \item[] Justification: We provide all of the relevant information in the main paper and appendices, including details of the models, datasets, computing resources, hyperparameters, and implementation of the algorithms. We also provide a detailed table of all hyperparameters used in the main experiments table.

\item {\bf Open access to data and code}
    \item[] Question: Does the paper provide open access to the data and code, with sufficient instructions to faithfully reproduce the main experimental results, as described in supplemental material?
    \item[] Answer: \answerNo{} %
    \item[] Justification: We will release the code once it has been carefully anonymized.

\item {\bf Experimental setting/details}
    \item[] Question: Does the paper specify all the training and test details (e.g., data splits, hyperparameters, how they were chosen, type of optimizer) necessary to understand the results?
    \item[] Answer: \answerYes{} %
    \item[] Justification: See the section on experiment results and also the appendices, including a detailed table of all hyperparameters used in the main experiments table.

\item {\bf Experiment statistical significance}
    \item[] Question: Does the paper report error bars suitably and correctly defined or other appropriate information about the statistical significance of the experiments?
    \item[] Answer: \answerYes{} %
    \item[] Justification: All our experiments were run with multiple seeds, and reported with error bars showing the standard deviation across those seeds.

\item {\bf Experiments compute resources}
    \item[] Question: For each experiment, does the paper provide sufficient information on the computer resources (type of compute workers, memory, time of execution) needed to reproduce the experiments?
    \item[] Answer: \answerYes{} %
    \item[] Justification: See the appendices, in the subsection Computing Resources.
    
\item {\bf Code of ethics}
    \item[] Question: Does the research conducted in the paper conform, in every respect, with the NeurIPS Code of Ethics \url{https://neurips.cc/public/EthicsGuidelines}?
    \item[] Answer: \answerYes{} %

\item {\bf Broader impacts}
    \item[] Question: Does the paper discuss both potential positive societal impacts and negative societal impacts of the work performed?
    \item[] Answer: \answerYes{} %
    \item[] Justification: We have discussed this point in the conclusion.
    
\item {\bf Safeguards}
    \item[] Question: Does the paper describe safeguards that have been put in place for responsible release of data or models that have a high risk for misuse (e.g., pre-trained language models, image generators, or scraped datasets)?
    \item[] Answer: \answerNA{} %
    \item[] Justification: We do not believe our work poses a high risk for misuse.

\item {\bf Licenses for existing assets}
    \item[] Question: Are the creators or original owners of assets (e.g., code, data, models), used in the paper, properly credited and are the license and terms of use explicitly mentioned and properly respected?
    \item[] Answer: \answerYes{} %
    \item[] Justification: We have cited the three datasets we used, as well as the machine learning framework Pytorch and the DP framework Opacus.

\item {\bf New assets}
    \item[] Question: Are new assets introduced in the paper well documented and is the documentation provided alongside the assets?
    \item[] Answer: \answerNA{} %

\item {\bf Crowdsourcing and research with human subjects}
    \item[] Question: For crowdsourcing experiments and research with human subjects, does the paper include the full text of instructions given to participants and screenshots, if applicable, as well as details about compensation (if any)? 
    \item[] Answer: \answerNA{} %

\item {\bf Institutional review board (IRB) approvals or equivalent for research with human subjects}
    \item[] Question: Does the paper describe potential risks incurred by study participants, whether such risks were disclosed to the subjects, and whether Institutional Review Board (IRB) approvals (or an equivalent approval/review based on the requirements of your country or institution) were obtained?
    \item[] Answer: \answerNA{} %

\item {\bf Declaration of LLM usage}
    \item[] Question: Does the paper describe the usage of LLMs if it is an important, original, or non-standard component of the core methods in this research? Note that if the LLM is used only for writing, editing, or formatting purposes and does \emph{not} impact the core methodology, scientific rigor, or originality of the research, declaration is not required.
    \item[] Answer: \answerNA{} %

\end{enumerate}

\end{document}